\documentclass[11pt,a4paper]{article}

\usepackage[utf8]{inputenc}
\usepackage[T1]{fontenc}
\usepackage{lmodern}
\usepackage{microtype}

\usepackage[margin=2.5cm]{geometry}

\usepackage{amsmath}
\usepackage{amssymb}
\usepackage{amsfonts}
\usepackage{amsthm}
\usepackage{bm}
\usepackage{mathtools}

\usepackage{amsmath,amsfonts,bm}

\def\eqref#1{equation~\ref{#1}}

\def\1{\bm{1}}

\def\mN{{\bm{N}}}

\DeclareMathAlphabet{\mathsfit}{\encodingdefault}{\sfdefault}{m}{sl}
\SetMathAlphabet{\mathsfit}{bold}{\encodingdefault}{\sfdefault}{bx}{n}

\DeclareMathOperator*{\argmax}{arg\,max}

\usepackage{thmtools}
\usepackage{thm-restate}

\newtheorem{lemma}{Lemma}

\newtheorem{definition}{Definition}
\newtheorem{corollary}{Corollary}

\usepackage{graphicx}
\usepackage{subcaption}
\usepackage[table]{xcolor}
\usepackage{booktabs}
\usepackage{multirow}
\usepackage{array}

\usepackage{algorithm}
\usepackage{algpseudocode}

\usepackage{enumitem}
\usepackage{placeins}
\usepackage[most]{tcolorbox}
\usepackage[normalem]{ulem}
\usepackage{url}

\usepackage[round,authoryear]{natbib}
\usepackage[
    colorlinks=true,
    linkcolor=blue,
    citecolor=blue,
    urlcolor=blue
]{hyperref}

\newcommand{\RMV}[1]{}

\newcommand{\NOTE}[1]{}

\providecommand{\mN}[1]{}
\renewcommand{\mN}[1]{}

\title{
    \textbf{On the Calibration of Crowdsourcing Aggregators}
}

\date{}

\author{\centering
Gabriel Singer$^{1,4}$ \and
Reda Arab$^{2}$ \and
Maïssa Frikha$^{1}$ \and
Samuel Gruffaz$^{3}$ \and Vo Van Olivier$^{4}$ \and
Nicolas Vayatis$^{1}$ \and 
Argyris Kalogeratos$^{1}$
\\[0.5em]
$^{1}$ ENS Paris-Saclay, Université Paris-Saclay, France \\
$^{2}$ CentraleSupélec, Université Paris-Saclay, France \\
$^{3}$ Tampere University, Finland
\\
$^{4}$ SNCF, Paris
}

\usepackage{amssymb,mathtools}
\usepackage{booktabs}
\usepackage{threeparttable}
\usepackage[scaled=0.95]{inconsolata}   
\makeatletter
\renewcommand\paragraph{\@startsection{paragraph}{4}{\z@}%
  {0.6ex \@plus 0.2ex \@minus .1ex}{-1em}{\normalsize\bf}}
\makeatother

\usepackage{hyperref}   
\begin{document}

\maketitle

\begin{abstract}
Calibration has been extensively studied for classifiers, but remains largely unexplored for crowdsourcing aggregation methods, which are 
increasingly 
used for making probabilistic decisions. To the best of our knowledge, we provide the first theoretical and empirical study of calibration in label aggregation. 
To our great surprise, we find that widely used aggregation methods can be severely miscalibrated, and in opposite directions: Majority Vote is underconfident while model-based aggregators are overconfident. We prove that the Bayesian aggregator is perfectly calibrated under its data-generating model. We also show that exact calibration of weighted majority votes requires highly restrictive conditions, which helps explain why Majority Vote is generally miscalibrated in realistic crowds. 
Empirically, we study calibration across
$39$ human crowdsourcing datasets and LLM-based crowds. In both cases, we observe systematic miscalibration for
aggregation models such as Majority Vote, Dawid--Skene, GLAD, and CrowdFM. 
Finally, we propose a post-hoc calibration procedure that \emph{aggregates the aggregators}: for any aggregation function $\phi$, we introduce a transformation that first combines its predictive distribution with those of complementary aggregators, before applying temperature scaling. 
This simple and interpretable transformation makes temperature scaling more effective than when applied directly to $\phi$. It improves calibration, especially in the high-redundancy regime where there are many annotations per item, while generally preserving or even improving accuracy.
\end{abstract}

\section{Introduction}
\label{sec:introduction}
Modern machine learning relies on large and carefully labeled datasets, and crowdsourcing has become a standard way to collect such labels \citep{ibrahim2025learning}:
each item is sent to several annotators,
and their responses are then aggregated into a single prediction.
This approach makes data annotation scalable,
but it also introduces noise.
Annotators may have different levels of expertise,
different biases,
and different error patterns. Historically, crowdsourcing has mainly relied on human annotators.
This setting is now changing.
Large language models (LLMs) are increasingly used as automatic annotators
\citep{Gilardi_2023}.
Recent studies show that LLMs can match or even outperform human crowd workers
on some annotation tasks \citep{plaza-del-arco-etal-2024-wisdom}.

Most crowdsourcing methods focus on recovering an accurate label from multiple noisy annotations, and their theoretical properties have been extensively studied~\citep{ibrahim2025learning}. However, most aggregation methods also return a probability distribution over the possible classes~\citep{whitehill2009glad, crowdfm}, which naturally raises the question of their calibration.

Calibrated probabilities are important when the output of an aggregation model is interpreted as a risk, and ensures that estimated probabilities correspond closely to observed frequencies. Poorly calibrated predictions can therefore lead to systematic over- or underestimation of risk and, consequently, to bad decisions.

Typically, in medicine, machine learning models can estimate the probability for a patient to develop a disease, experience a complication, or respond to a treatment, thereby helping to 
inform treatment decisions \citep{Martel2026, VanCalster2025}. Insufficiently calibrated models can lead to inappropriate and potentially harmful clinical decisions \citep{VanCalster2019Calibration}. Aggregation offers complementary opportunities to support clinical decision-making: studies have shown that
combining independent diagnoses can improve accuracy, both when aggregating physicians’ outputs \citep{Barnett2019} and when combining differential diagnoses generated by multiple LLMs \citep{Barabucci2024}. This motivates examining the calibration of aggregation methods in medical settings.

\begin{figure*}[t]
    \centering

    \begin{minipage}[b]{0.34\textwidth}
        \centering
        \includegraphics[
            width=\linewidth,
            keepaspectratio
        ]{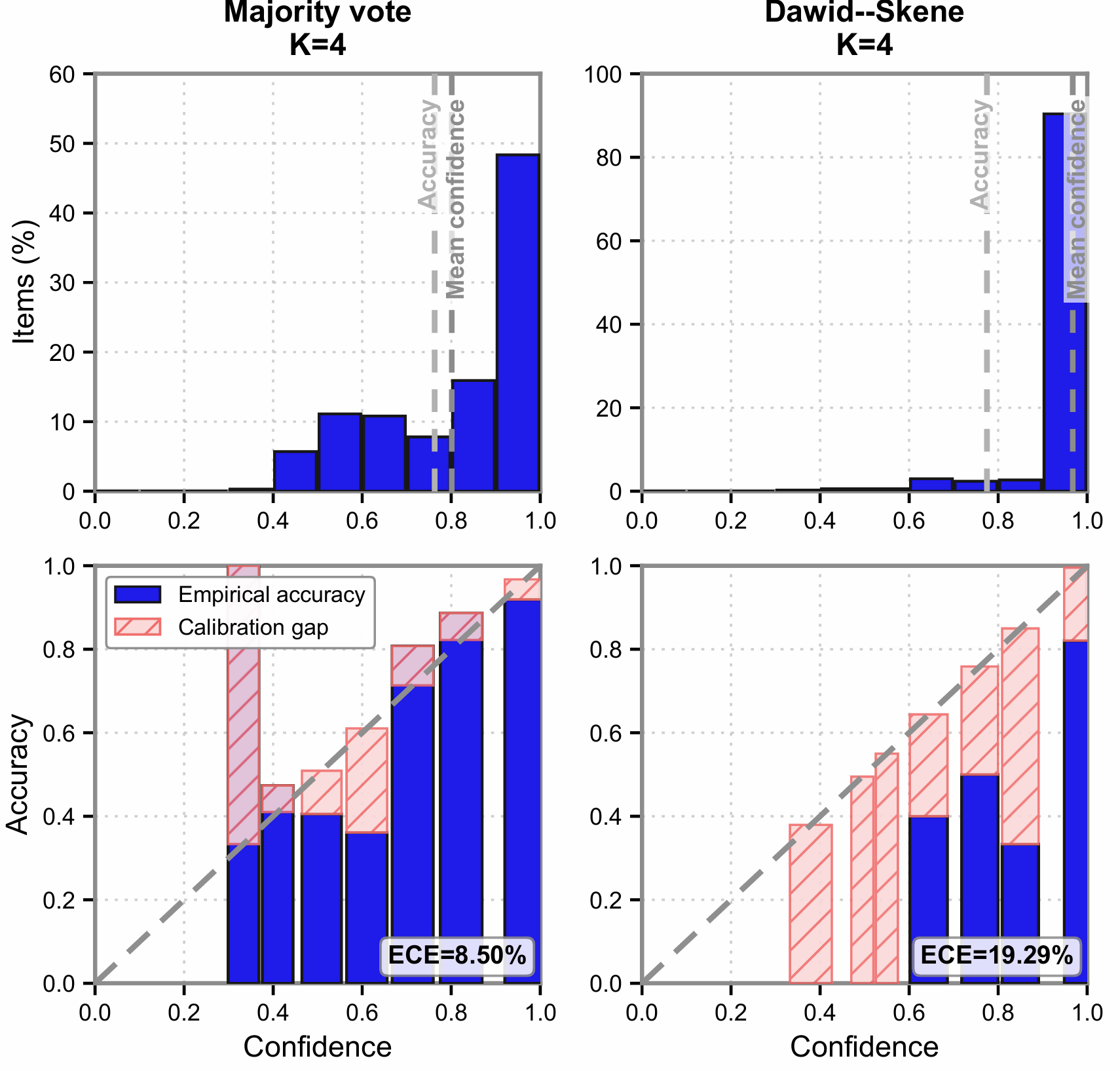}
        \caption*{\textbf{(a)} Adult dataset}
    \end{minipage}
    \hfill
    \begin{minipage}[b]{0.65\textwidth}
        \centering
        \includegraphics[
            width=\linewidth,
            keepaspectratio
        ]{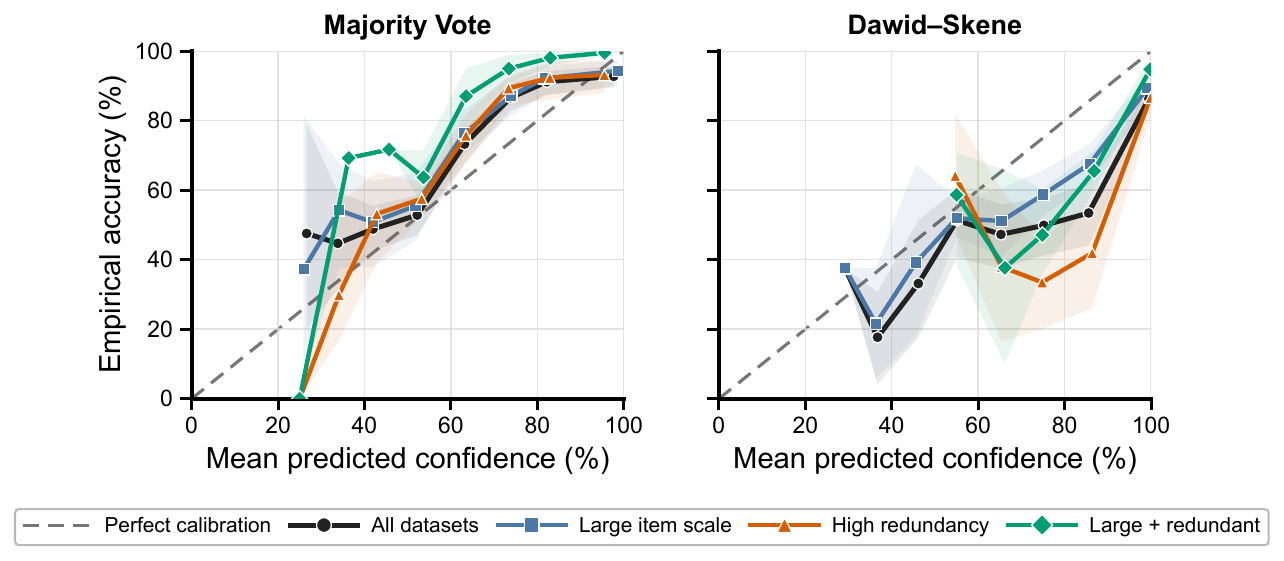}
        \caption*{\textbf{(b)} Average reliability over $39$ datasets}
    \end{minipage}

    \caption{
        Calibration of Majority Vote and Dawid--Skene.
        \textbf{(a)} Confidence histograms (top) and reliability diagrams
        (bottom) on the Adult dataset~\citep{wright2025aggregating}, with Expected Calibration Error (ECE, see Eq.~\eqref{ECE}) reported
        for each aggregator. $K$ is the number of classes.
        \textbf{(b)} Reliability diagrams averaged over $39$ datasets;
        points above the diagonal indicate \emph{underconfidence} and points below
        indicate \emph{overconfidence}. 
    }
    \label{fig:intro_calibration}
\end{figure*}

\paragraph{Contributions}
\begin{enumerate}[topsep=3pt, itemsep=2pt, parsep=0pt, leftmargin=1.5em]
    \item To the best of our knowledge, we provide the first empirical and theoretical study of calibration in crowdsourcing.
Our empirical study is done across $40$ real-world datasets and one specialized medical dataset. Two of them involve LLM crowds. 
To
our great surprise, we discover that Majority Vote, Dawid--Skene, GLAD, and CrowdFM models (see Section~\ref{sec:related-work} for the related work) are all severely miscalibrated: Majority Vote is \emph{underconfident} while model-based aggregators are \emph{overconfident} (see Figure~\ref{fig:intro_calibration}).
\item We characterize theoretically, in the multiclass setting, when crowd aggregation can be calibrated.
We first prove that the Bayesian aggregator is perfectly calibrated, and derive
necessary and sufficient conditions for asymptotic exact calibration of weighted majority vote. What is shown is that Majority Vote is generally asymptotically hard to calibrate.
\item We introduce a transformation of any aggregation function
$\phi$ that makes its probability output more effective under temperature scaling, and we wrap our approach in Algorithm~\ref{alg:general-Soup-ts}.
Compared to directly applying temperature scaling on $\phi$, our transformation improves the calibration of weighted-vote aggregators, and beyond those. For all our baselines, in the high-redundancy regime we improve calibration while preserving predictive performance.
\end{enumerate}
\section{Related work}\label{sec:related-work}
\paragraph{Learning from Crowd-Sourced Labels.}
Several label aggregation methods have been proposed in the literature.
Probabilistic approaches include Dawid--Skene~\citep{dawid1979maximum},
GLAD~\citep{whitehill2009glad}, and CrowdFM~\citep{crowdfm}.
These methods have mainly been studied through prediction error, identifiability,
sample complexity, and assumptions on annotator dependence~\citep{ibrahim2025learning}.
Weighted majority votes form another important family.
\citet{weightedvoteTheory} study the consistency of the
optimal weighted majority rule, while Crowd-Kit~\citep{Ustalov_2024} provides
practical implementations of weighted variants such as Wawa~\citep{Ustalov_2024} and ZBS~\citep{Ustalov_2024}.
\paragraph{Calibration.}
The calibration of probabilistic predictors has been extensively studied outside the crowdsourcing literature. \citet{guo2017calibrationmodernneuralnetworks} show that modern neural networks are poorly calibrated despite improved accuracy, and benchmark several post-hoc calibration techniques -- histogram binning, isotonic regression, Bayesian binning into quantiles, and Platt scaling -- introducing \emph{temperature scaling} as a simple, effective extension of the latter. Temperature scaling has since been applied well beyond image classification, notably to calibrate pre-trained language models~\citep{desai2020calibration}.  Prior work has documented limitations of temperature scaling -- namely its inability to correct classwise miscalibration~\citep{kull2019temperaturescalingobtainingwellcalibrated} or to guarantee a verified calibration error~\citep{verifiedUncertaintyCalibration} -- and, in line with these findings, we observe that it fails to correct the miscalibration of weighted-majority-vote-type aggregators. Rather than replacing temperature scaling itself, we address this limitation by acting on the aggregator being calibrated.

\paragraph{LLMs as annotators.}
Large Language Models (LLMs) have recently emerged as a scalable alternative to human annotators \citep{bhat-varma-2023-large}. Their use as labelers seems to be promising, especially in text classification~\citep{Gilardi_2023,ronningstad-etal-2024-gpt}. The interested reader can also find the review of \citet{tan2024largelanguagemodelsdata} about the use of LLMs for data annotation and synthesis.
\section{Problem setting: a probabilistic point of view}
The primary goal of label aggregation in crowdsourcing is to recover the unknown
ground-truth label of an item from multiple noisy annotations.

\paragraph{Preliminaries.}
Let $R\geq 1$ be the number of annotators,  let 
\(
Y_{1},\ldots,Y_{R}
\)
be $R$ random variables. We denote by $Y$ the unknown ground-truth label taking a value in \(
[K]:=\{0,\ldots,K-1\}\). We consider that \(
Y_{1},\ldots,Y_{R}
\) are independent conditionally to $Y$ and all taking values in
\(
[K]\).

A probabilistic aggregation model $\phi$ is a function that takes as input random variables and outputs a probability distribution
\(
\phi:[K]^R\longrightarrow \Delta_{K},
\)
where
\(
\Delta_{K}
:=
\big\{
q\in[0,1]^K:
\sum_{k\in[K]} q_k=1
\big\}
\).

The goal of the 
model $\phi$ is to estimate the posterior
distribution
\(
\left(
\mathbb{P}\!\left(
Y=k\mid Y_{1},\ldots,Y_{R}
\right)
\right)_{k\in[K]}.
\)
The final aggregated label is then obtained by
\begin{equation}\label{aggregation}
Y_{R}^\phi
\in
\operatorname*{arg\,max}_{k\in[K]}
\phi_{k}(Y_{1},\ldots,Y_{R}).
\end{equation}
Such probabilistic aggregation models include classical methods such as
Dawid--Skene \citep{dawid1979maximum} and GLAD \citep{whitehill2009glad}. It also covers
more recent deep learning models such as CrowdFM, a foundation model pretrained on synthetic data \cite{crowdfm}. Even majority voting can naturally be interpreted in a probabilistic way: for any $r\in\{1,\ldots,R\}$, set  
\(
\phi_{k}(Y_{1},\ldots,Y_{R})
=
\frac{1}{R}
\sum_{r=1}^{R}
\mathbb{I}\{Y_{r}=k\}.
\)
\paragraph{Probabilistic calibration.}
Unlike the confidence-based notion 
\citep{guo2017calibrationmodernneuralnetworks}, which only calibrates the predicted class confidence
\(
c_\phi(R)
:=
\max_{k\in[K]}
\phi_k(Y_1,\ldots,Y_R),
\)
we consider calibration of the full  distribution
\(
\phi(Y_{1},\ldots,Y_{R}).
\)
This distinction matters for \(K\geq 3\), since a predictor may be calibrated
with respect to its top$_k$-class confidence while remaining miscalibrated at the
level of the full distribution as shown in \cite{pmlr-v89-vaicenavicius19a}. Note that, being calibrated in terms of the full distribution implies to be also calibrated with respect to the subset of its top$_k$-class. Readers interested in a detailed example may refer to Appendix~\ref{app:mv-example}.

\begin{definition} (Calibrated model)
  We say that $\phi$ is calibrated if and only if for any $k \in [K]$,\begin{equation} \label{perfectcalibrationequation}\mathbb{P}\!\left( Y=k \mid \phi(Y_{1},\ldots,Y_{R}) \right) = \phi_{k}(Y_{1},\ldots,Y_{R}) \qquad \text{a.s.} \end{equation} 
\end{definition}

\paragraph{Calibration Error.}
The condition~\eqref{perfectcalibrationequation} must hold almost surely, in practice, this condition is impossible to check, so one instead considers their average difference.
This naturally leads to the calibration error
\citep{verifiedUncertaintyCalibration,pmlr-v89-vaicenavicius19a,ecconsistencyestimator}.
Let
\(
Z
=
\left(
\mathbb{I}\{Y=1\},
\ldots,
\mathbb{I}\{Y=K\}
\right)
\in \{0,1\}^K
\)
denote the one-hot encoding of the ground-truth label, and let
\(
\Phi_R
:=
\phi_R(Y_{1},\ldots,Y_{R})
\in \Delta_K.
\)

\begin{definition}
For $p \in [1,+\infty[$, the $L^p$ calibration error of $\phi_R$ is defined as
\[
\operatorname{CE}_p(\phi_R)
:=
\left(
\mathbb{E}\left[
\left\|
\mathbb{E}\!\left[
Z \mid \Phi_R
\right]
-
\Phi_R
\right\|_p^p
\right]
\right)^{1/p}.
\]
\end{definition}
For \(p=1\), this quantity corresponds to the \(L^1\)  calibration
error, which has been studied theoretically in
\cite{verifiedUncertaintyCalibration,pmlr-v89-vaicenavicius19a,
ecconsistencyestimator}.
 Following \cite{samworth} and \cite{singer2026optimalfairaggregationcrowdsourced},
we study aggregation in the regime \(R\to\infty\).

\section{Theoretical Analysis of the Calibration for Weighted Votes}\label{sectionweightedtheory}

Here, we study theoretically the calibration of weighted majority votes.
Among the most famous one, you have the Bayesian vote, which minimizes the classification error, and the standard Majority Vote.
In this section, we identify annotator number $i$ with their response random variable $Y_i$
, thus this can be applied any type of annotator, human or LLM.
\paragraph{The Bayesian vote is perfectly calibrated.}
The first fact is that Bayesian vote, that can be written as a weighted majority vote see \citep{singer2026optimalfairaggregationcrowdsourced}, is perfectly calibrated.

\begin{restatable}{proposition}{PropBayesCalibrated}\label{bayesCalibration}
    The Bayesian vote defined as 
    \(
\phi^\star
=
\operatorname*{arg\,min}_{\phi}
\mathbb P\!\left(Y^\phi_R\neq Y\mid Y_1,...,Y_R\right),
\)
    is calibrated in the full distribution sense,top$_k$ class sens and $L^{p}$ sens for any $p\geq 1$.
\end{restatable}
Following \cite{pmlr-v89-vaicenavicius19a}, it appears that any aggregation method defined by the mean of a function $h$ as follows:
$\phi_{h} : (Y_1,...,Y_R)\mapsto \left(\mathbb{P}(Y=k \mid h(Y_1,...,Y_R))\right)_{k\in [K]}$ is also calibrated.  
\paragraph{Analysis of weighted Majority Vote calibration.}
For
$\omega^{(R)}=(\omega_1^{(R)},\ldots,\omega_R^{(R)})\in\Delta_R$, we say that $\Phi_R$ is  weighted majority vote if one can right
\(
    \Phi_R
    :=
   \left( \sum_{r=1}^R \omega_r^{(R)} \mathbb{I}\{Y_r=k\}\right)_{k\in [K]}
\)
Let $\mathcal G$ denote the class of sequences weighted majority vote $(\Phi_R)_{R\geq1}$
such that
\begin{equation}\label{Gcond}
    \lim_{R\to \infty}\Big\| \omega^{(R)}\Big\|_{2} =0
    \text{ and }
    \Phi_R\mid Y=j
    \xrightarrow{\mathbb P} A^{(j)}
    \quad\text{for every } j\in [K].
\end{equation}
where $A^{(1)},\ldots,A^{(K)}\in\Delta_K$ are deterministic and pairwise
distinct. The assumption that $(A^{(j)})_{j\in [k]}$ is deterministic corresponds to the standard
law-of-large-numbers regime for crowdsourcing, it includes in particular the classical Condorcet Jury Theorem \citep{ condorcetmonotonheterogeneindep}. 
One can show, see Appendix~\ref{app:proofs}, that~\eqref{Gcond} defines uniquely the limite aggregator as 
$\Phi:=\lim_{R\to \infty}\Phi_R=A^{Y}.$
Condition~\eqref{Gcond} only requires the crowd average accuracy
to stabilize at $\infty$. So, it allows a lot of different heterogeneous annotators configurations. 
\begin{restatable}{proposition}{PropGNonempty}\label{prop:g-nonempty}
The class $\mathcal G$ is nonempty. Majority
Vote belongs to $\mathcal G$ and 
Wawa also belongs to $\mathcal G$. 
\end{restatable}

\begin{restatable}{theorem}{ThmWMVCalibration}\label{thm:wmv-calibration}
For any $\Phi\in \mathcal{G}$ and $p\geq 1:$ the following conditions are equivalent: 
if
\begin{equation}\label{calibMVasymptotic}
    \lim_{R\to \infty}\operatorname{CE}_{p}(\phi_R)=0
\end{equation}
    Then, for any $(k,j)\in [K]^{2}$: \begin{equation}\label{AsymtoticcalibrationMV}
        \delta_{k,j}=\lim_{R\to \infty}\sum_{1\leq r\leq R}\omega^{(R)}_r\mathbb{P}\left(Y_r=j\mid Y=k\right),
    \end{equation}
    and reciprocally.
\end{restatable} 
Where $\delta_{k,j}$ denotes the Kronecker delta, i.e., $\delta_{k,j}=1$ if $k=j$ and $\delta_{k,j}=0$ otherwise.

If an aggregator is not asymptotically perfectly efficient it cannot be calibrated.
\begin{restatable}{corollary}{CorInaccurate}\label{cor:inaccurate-not-calibrated}
Let $\phi\in \mathcal{G},$
if $\displaystyle \limsup_{R\to \infty}\mathbb{P}\left(Y_R^{\phi}\neq Y\right)>0,$
   then $\phi$ is not asymptotically calibrated.
\end{restatable}

\paragraph{Interpretation.}
The condition~\eqref{AsymtoticcalibrationMV} maybe interpreted as follows: consider a growing crowd with respect to the number $R$ of annotators. If for each class $k$, the average accuracy of the crowd is close to be perfect  then, on average it's calibration error is close to $0$. 
So taking the contraposition of the direct implication it shows you that if you do not have a very accurate crowd, you are not calibrated. This interpretation is explicitly proven in
Corollary~\ref{cor:inaccurate-not-calibrated}. 
\section{Temperature Scaling for probabilistic aggregation functions}
\label{sec:crowd_ts}

We adopt temperature scaling as our calibration method for its simplicity and single-parameter form, which makes it more robust to overfitting on limited calibration data than more flexible alternatives such as Dirichlet calibration~\citep{kull2019temperaturescalingobtainingwellcalibrated}; moreover, temperature scaling is order-preserving, allowing us to recalibrate an aggregator's confidence without altering its underlying decisions. 

\paragraph{Calibration and post-hoc calibration.}

A probabilistic predictor is calibrated when its predicted probabilities match
empirical frequencies. For instance, predictions made with confidence $0.8$
should be correct about $80\%$ of the time. Post-hoc calibration aims to correct predicted probabilities without retraining
the original model. A standard method is temperature scaling
\citep{guo2017calibrationmodernneuralnetworks}, which transforms the estimated probability of the item i to be in class k : $\widehat p_{ik}$. For $T>0,$
\(
\widetilde p_{ik}(T)
=
{\widehat p_{ik}^{1/T}} / {\sum_{\ell=1}^{K}\widehat p_{i\ell}^{1/T}}.
\)
The temperature is fitted on a separate calibration set by minimizing the
negative log-likelihood. If we denote $Y_i$ the true label,
\(
T^\star
:=
\arg\min_{T>0}
-\sum_{i=1}^{n_{\mathrm{cal}}}
\log \widetilde p_{iY_i}(T).
\)

\paragraph{Temperature scaling for label aggregation.}
Let $\phi:[K]^R\to\Delta_K$ be a probabilistic aggregation function and let
\(
Q_i^\phi=\phi(Y_{i,1},\ldots,Y_{i,R}).
\)
Temperature scaling transforms $Q_i^\phi$ through a single scalar parameter $\frac{1}{T}$,
preserving the class ranking while only rescaling its log-ratios.
More flexible calibrators, such as Dirichlet calibration \citep{kull2019temperaturescalingobtainingwellcalibrated}, increase the
expressiveness of this transformation across classes.

\paragraph{A crowdsourcing-aware correction.}
Given a target aggregator $\phi$ and a 
$\mathcal{R}=\{\psi_1,\ldots,\psi_M\}$ a fixed set of reference
aggregators. 
Our intuition is that aggregators with complementary calibration profiles can
compensate for each other before temperature scaling.
The role of $\phi$ is to preserve the predictive accuracy of the target aggregator.
By design the reference set $\mathcal R$ contains any collection of
aggregators with complementary calibration profiles. Temperature scaling is then fitted on $Q_i^{\mathrm{Soup},\phi}$.
\begin{algorithm}[h!]\small
\caption{Crowdsourcing-Aware Temperature Scaling}
\label{alg:general-Soup-ts}
\begin{algorithmic}[1]

\Require Fixed set of aggregators
$\mathcal{R}=\{\psi_1,\ldots,\psi_M\}$,

\ \ Calibration set
\(
\mathcal D_m^{\mathrm{cal}}
=
\left\{
\left(
Y_i,
\{Q_i^\psi\}_{\psi\in\mathcal{R}}
\right)
\right\}_{i=1}^{m},
\)

\ \ Predictive distributions
$\{Q_j^\psi\}_{\psi\in\mathcal{R}}$
for the items $j$ to calibrate.

\Statex \!\!\!\!\!\!\!\!\textbf{Output:}
Calibrated predictive distributions
$\{\widetilde Q_j^{\mathrm{Soup},\phi}\}_j$.

\vspace{0.2em}
\hrule
\vspace{0.5em}

\State
$\operatorname{Soup}
\left(
\{Q^\psi\}_{\psi\in\mathcal R}
\right)
\gets
\displaystyle
\frac{1}{|\mathcal R|}
\sum_{\psi\in\mathcal R} Q^\psi$
\Comment{Define the aggregation operator}

\For{$i=1,\ldots,m$}
    \State
    $Q_i^{\mathrm{Soup},\phi}
    \gets
    \operatorname{Soup}
    \left(
    \{Q_i^\psi\}_{\psi\in\mathcal R}
    \right)$
    \Comment{Build calibration Soup}
\EndFor

\State
$\widehat T_m^\phi
\gets
\displaystyle
\arg\min_{T>0}
-\frac{1}{m}
\sum_{i=1}^{m}
\log
\left[
\mathcal T_T
\left(
Q_i^{\mathrm{Soup},\phi}
\right)
\right]_{Y_i}$
\Comment{Fit temperature on calibration data}

\For{each item $j$ to calibrate}
    \State
    $Q_j^{\mathrm{Soup},\phi}
    \gets
    \operatorname{Soup}
    \left(
    \{Q_j^\psi\}_{\psi\in\mathcal R}
    \right)$
    \Comment{Build prediction Soup}

    \State
    $\widetilde Q_j^{\mathrm{Soup},\phi}
    \gets
    \mathcal T_{\widehat T_m^\phi}
    \left(
    Q_j^{\mathrm{Soup},\phi}
    \right)$
    \Comment{Apply learned temperature}
\EndFor

\State \Return
$\{\widetilde Q_j^{\mathrm{Soup},\phi}\}_j$

\end{algorithmic}
\end{algorithm}

\section{Experiments}
\subsection{Datasets and methods presentations}
\paragraph{Datasets annotated by humans.}
Regarding human annotators, we evaluate the methods on $39$ real-world crowdsourcing datasets. 
Including image annotation, sentiment analysis,
natural-language tasks, and toxicity detection. Among others, we have: LabelMe \citep{russell2008labelme}, the NLP crowdsourcing
tasks of \citet{snow2008cheap}, Jigsaw toxicity data
\citep{borkan2019nuanced,jigsaw_unintended_bias_2023}, and CIFAR-10H
\citep{peterson2019humanuncertaintymakesclassification}.
Additional crowdsourcing benchmarks are taken from  Crowd-Kit \citep{Ustalov_2024} and CrowdFM \citep{crowdfm}.
We consider four regimes: all datasets ($39$ datasets), large-scale datasets with at least $1{,}000$ items ($21$ datasets), high-redundancy datasets with at least $10$ annotations per item on average ($19$ datasets), and their intersection ($6$ datasets).

\paragraph{Dataset annotated by LLMs.} 

LLMs were evaluated on two datasets. 
Firstly, to compare the calibration of aggregated human and LLM predictions, we use the crowd-annotated version of the PASCAL RTE-1 benchmark \citep{rte,snow2008cheap}, one of the $39$ real-world crowdsourcing datasets annoted by humans. This is a binary textual entailment task which requires determining whether a hypothesis can be inferred from a given textual premise. The dataset contains $800$ Text-Hypothesis pairs, each with a reference label and each pair annotated by ten human judgments collected from a pool of $164$ non-expert annotators. We obtain these existing human annotations from the data released with \citet{wright2025aggregating}\footnote{PASCAL RTE-1 accessed via: \\\url{https://github.com/copenlu/aggregating-crowd-annotations-ood/tree/main/data/RTE}} and collect LLM predictions on the same pairs. We then apply the same aggregation rules to human and LLM responses and evaluate the calibration of the resulting predictions against the reference labels.

As discussed in the introduction, the importance of calibrated predictions for clinical decision-making motivates examining the calibration of aggregation methods in medical settings. For that purpose, we decided to evaluate LLMs on a second dataset, a medical benchmark: MMLU Clinical knowledge \citep{mmlu}.
This benchmark evaluates clinical knowledge through a test set of $265$ multiple-choice questions, each with four answer options.\footnote{MMLU dataset accessed via Hugging Face: \\\url{https://huggingface.co/datasets/cais/mmlu/viewer/clinical_knowledge}.} Results on this dataset are reported in \ref{app:exp_results}.

\paragraph{Design of the experiments.}
We evaluate eleven models accross three open-source models families (Qwen, Mistral, and Gemma) and one closed-source model (from OpenAI). Because we run the open-source models on an Apple M5 integrated GPU, we restrict them to fewer than $10$ billion parameters. 
Table~\ref{tab:ensemble-models-params} in appendix lists the different models, with HuggingFace's IDs and generation parameters.

To reproduce the human evaluation, we apply the same aggregation rules, treating $R$ distinct models as annotators.

\subsection{Miscalibration of aggregators}
We follow the standard empirical Expected Calibration Error (ECE)~\cite{guo2017calibrationmodernneuralnetworks}. Given $M$ confidence bins $(B_m)_{m=1}^M$,

\begin{equation}\label{ECE}
    \mathrm{ECE}
=
\sum_{m=1}^{M}
\frac{|B_m|}{n}
\left|
\operatorname{acc}(B_m)-\operatorname{conf}(B_m)
\right|,
\end{equation}

where $\operatorname{acc}(B_m)$ and $\operatorname{conf}(B_m)$ are the empirical accuracy and mean confidence in bin $B_m$. Lower values indicate better calibration.
\begin{figure*}[t]
    \centering

    \begin{minipage}[b]{0.45\textwidth}
        \centering
        \includegraphics[
            width=\linewidth,
            keepaspectratio
        ]{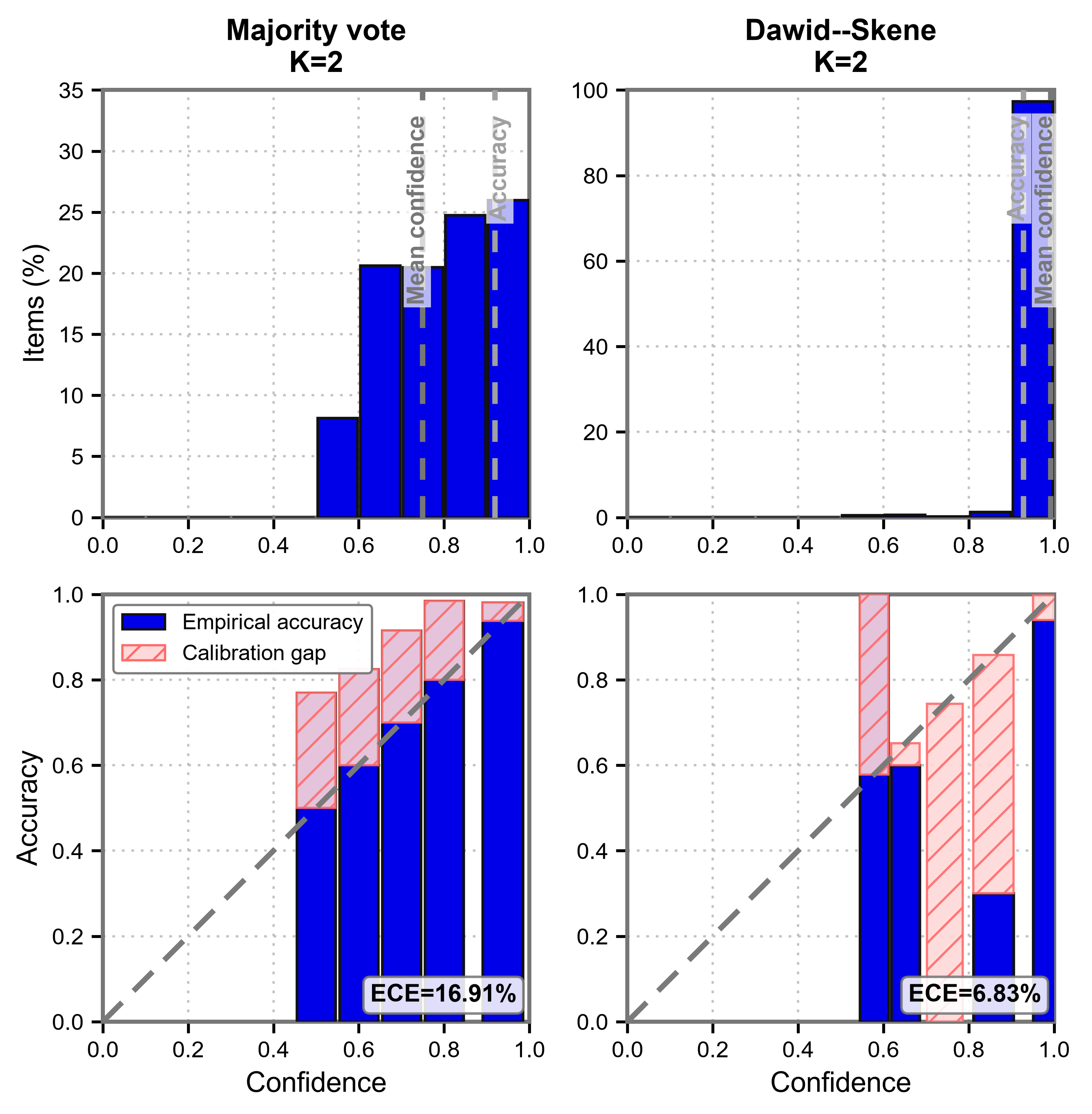}
        \caption*{\textbf{(a)} RTE dataset - Human annotated}
    \end{minipage}
    \hfill
    \begin{minipage}[b]{0.45\textwidth}
        \centering
        \includegraphics[
            width=\linewidth,
            keepaspectratio
        ]{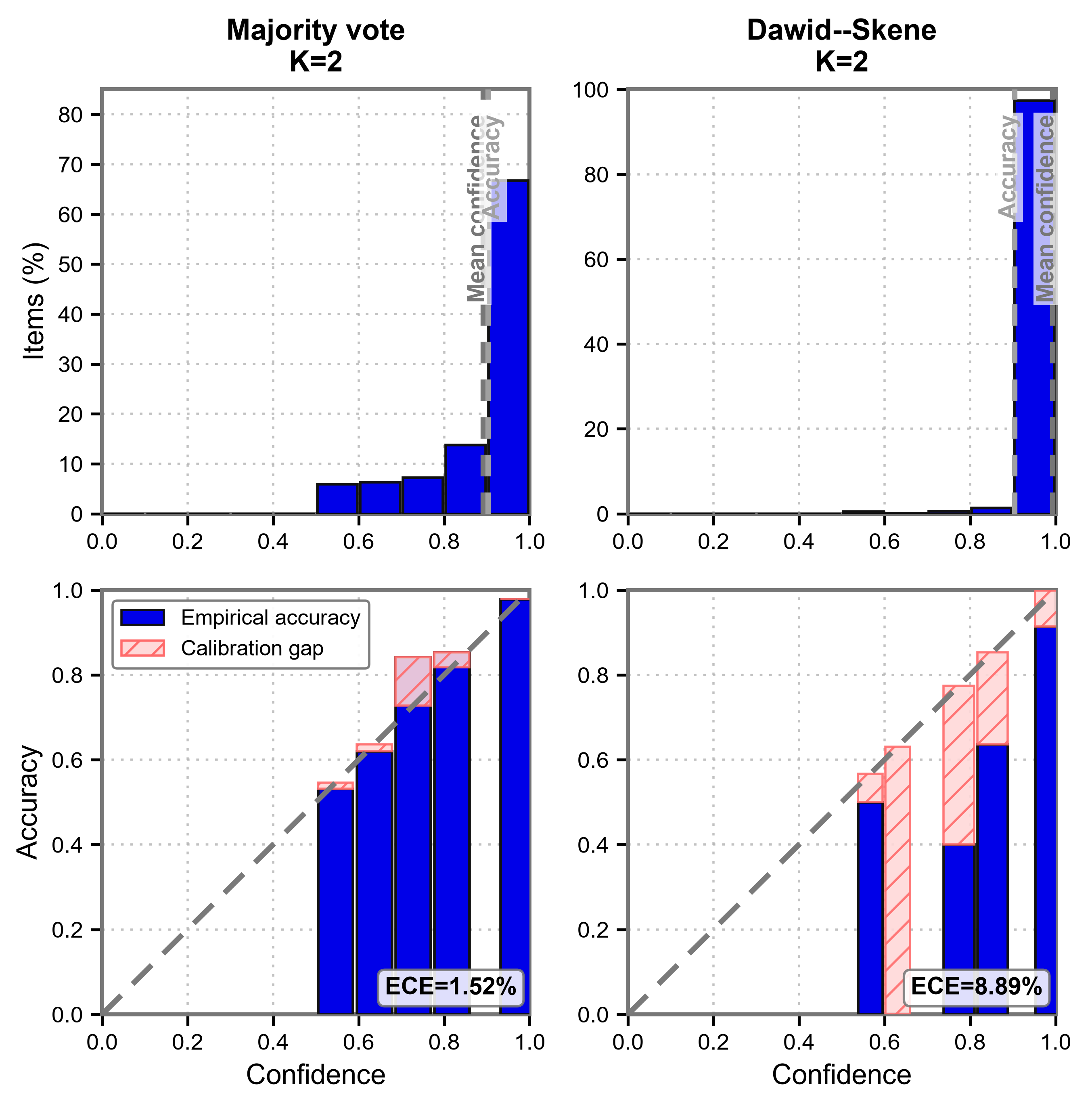}
        \caption*{\textbf{(b)} RTE dataset - LLM annotated}
    \end{minipage}

    \caption{
        Confidence histograms (top) and reliability histograms(bottom) of Majority Vote, DS aggregators on \textbf{(a)} the RTE human crowd dataset and on the \textbf{(b)} RTE LLM crowd dataset.
    }
    \label{fig:calibration_mv_ds}
\end{figure*}

Both for human and LLM annotators, Majority Vote shows underconfidence, while Dawid--Skene tend to be overconfident (Figure~\ref{fig:calibration_mv_ds}). On RTE dataset and the LLM crowd dataset, Dawid–Skene reaches respectively ECE of $6.8\%$ and $8.9\%$, while MV reaches respectively ECE of $16.9\%$ and $1.5\%$.

\subsection{Temperature Scaling and Soup}
\textbf{In our
experiments}, we use
\(
\mathcal R = \{\mathrm{MV},\mathrm{DS}\}
\)
because the empirical analysis in Figure~\ref{fig:calibration_mv_ds} reveals a clear
complementarity between Majority Vote and Dawid--Skene: Majority Vote tends
to be under-confident, whereas Dawid--Skene tends to be over-confident. 
Therefore, in practice with our algorithm \ref{alg:general-Soup-ts}, we set
\[
Q_i^{\mathrm{Soup},\phi}
:=
\frac{1}{3}
\left(
Q_i^{\mathrm{MV}}
+
Q_i^{\mathrm{DS}}
+
Q_i^\phi
\right).
\]

\begin{table*}[!hbtp]
\centering
\scriptsize
\setlength{\tabcolsep}{4pt}
\begin{tabular}{lcrrrrrrrrrr}
\toprule

&&\multicolumn{9}{c}{\textbf{Weighted  vote aggregation methods}} \\[1mm]

& &
\multicolumn{3}{c}{\cellcolor{blue!10}\textbf{MV}}
& \multicolumn{3}{c}{\cellcolor{blue!10}\textbf{Wawa}}
& \multicolumn{3}{c}{\cellcolor{blue!10}\textbf{ZBS}} \\

\cmidrule(lr){3-5}
\cmidrule(lr){6-8}
\cmidrule(lr){9-11}

Dataset regime & $N$
& MV & TS & Soup+TS
& Wawa & TS & Soup+TS
& ZBS & TS & Soup+TS \\
\midrule

Human datasets
& 39
& \cellcolor{blue!4}11.98
& \cellcolor{blue!4}10.64
& \cellcolor{blue!4}\textbf{8.48}
& \cellcolor{blue!4}10.67
& \cellcolor{blue!4}10.45
& \cellcolor{blue!4}\textbf{8.24}
& \cellcolor{blue!4}10.68
& \cellcolor{blue!4}10.39
& \cellcolor{blue!4}\textbf{8.20} \\

LLM crowd & 1
& \cellcolor{blue!4}9.24
& \cellcolor{blue!4}\textbf{5.88}
& \cellcolor{blue!4}7.30
& \cellcolor{blue!4}9.22
& \cellcolor{blue!4}\textbf{6.11}
& \cellcolor{blue!4}7.27
& \cellcolor{blue!4}8.86
& \cellcolor{blue!4}\textbf{6.11}
& \cellcolor{blue!4}7.27 \\

\midrule
&&\multicolumn{9}{c}{\textbf{Model-based aggregation methods}} \\[1mm]

& &
\multicolumn{3}{c}{\textbf{DS}}
& \multicolumn{3}{c}{\textbf{GLAD}}
& \multicolumn{3}{c}{\textbf{CrowdFM}} \\

\cmidrule(lr){3-5}
\cmidrule(lr){6-8}
\cmidrule(lr){9-11}

Dataset regime & $N$
& DS & TS & Soup+TS
& GLAD & TS & Soup+TS
& CrowdFM & TS & Soup+TS \\
\midrule

Human datasets
& 39
& 13.58 & \textbf{5.88} & 7.17
& 13.69 & \textbf{6.06} & 6.27
& 12.50 & \textbf{4.26} & 4.62 \\

LLM crowd & 1
& 17.13
& 12.02
& \textbf{7.31}
& 14.12
& 8.94
& \textbf{7.79}
& 11.87
& \textbf{6.74}
& 7.23 \\
\bottomrule
\end{tabular}
\caption{Expected Calibration Error (ECE) scores.}
\label{tab:raw_ts_Soup_ts_groups}
\end{table*}

\begin{figure*}[h!]
    \centering

    \begin{subfigure}[htbp]{0.48\textwidth}
        \centering
        \includegraphics[width=\linewidth]{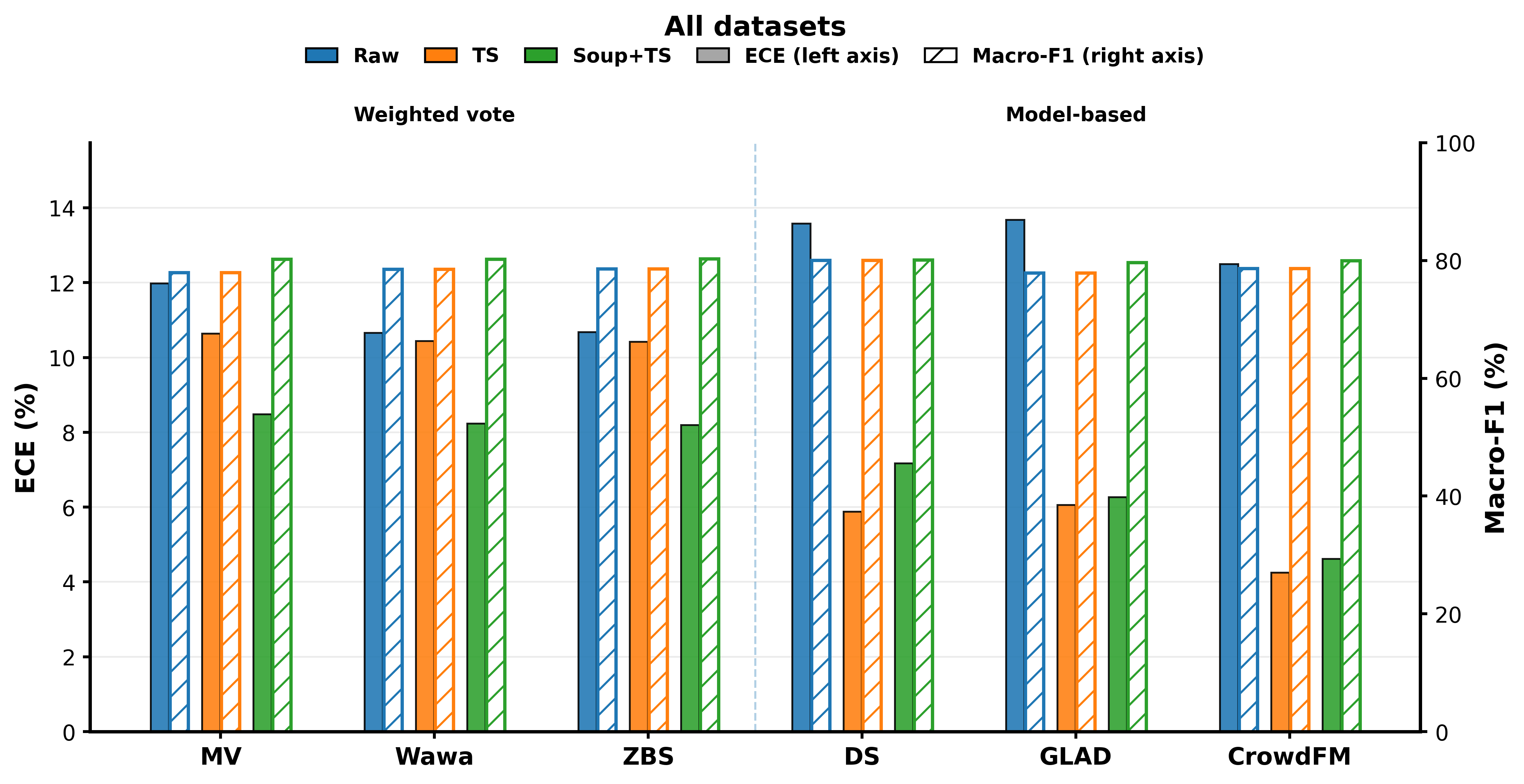}
        \caption{Human datasets}
        \label{fig:ece_f1_all_datasets}
    \end{subfigure}
    \hfill
    \begin{subfigure}[htbp]{0.48\textwidth}
        \centering
        \includegraphics[width=\linewidth]{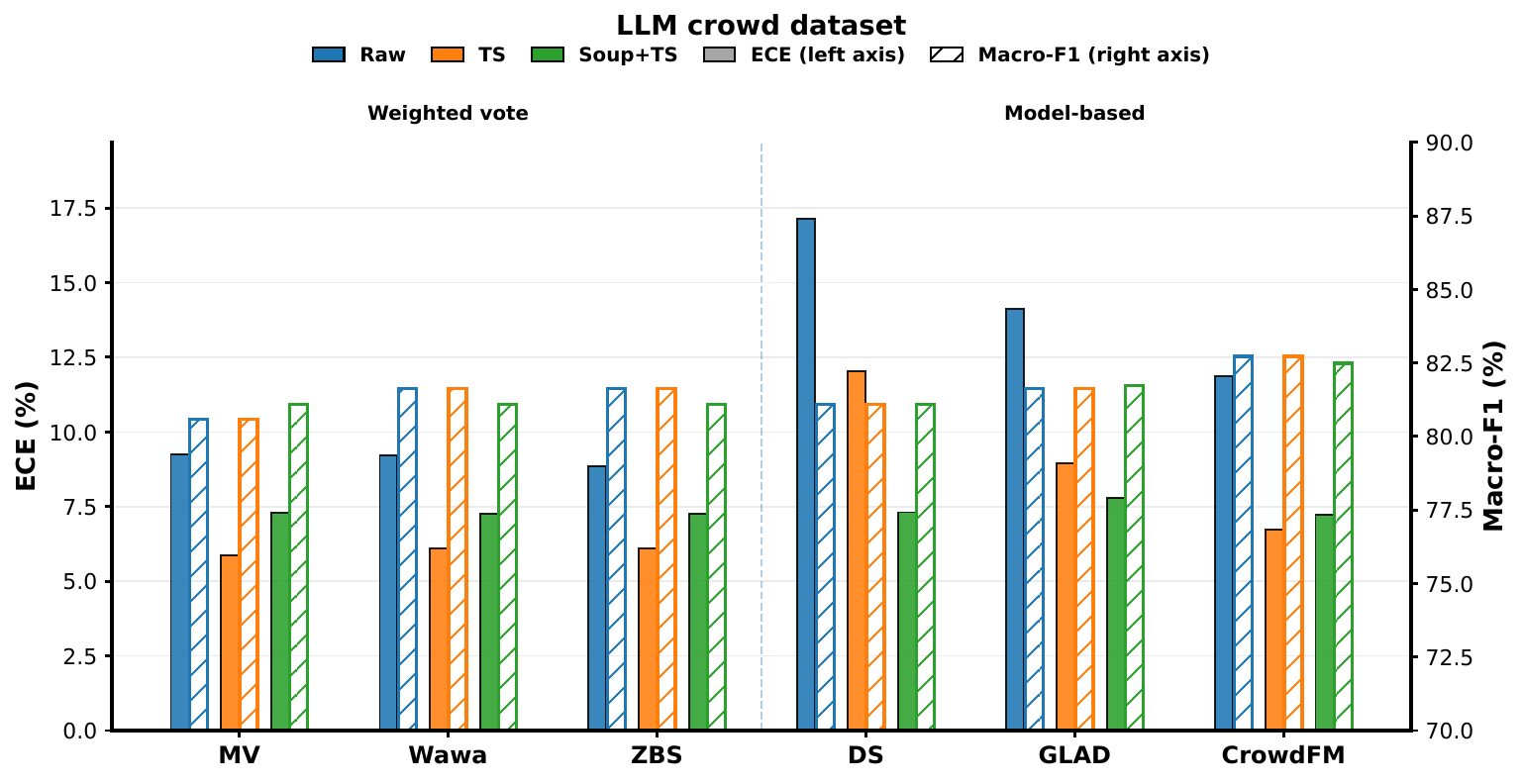}
        \caption{LLM crowd dataset}
        \label{fig:ece_f1_llm}
    \end{subfigure}

    \caption{\label{fig:ece_f1_comparison} ECE and Macro-F1 comparison for Raw, TS, and Soup+TS for human and llm datasets.}
    
\end{figure*}

\paragraph{Results}
Table~\ref{tab:raw_ts_Soup_ts_groups} shows  improvements in terms of ECE
in several settings. On human datasets, for weighted majority votes, our
Soup+TS  systematically decreases the ECE compared to
classic temperature scaling, while for model-based aggregators, temperature scaling
alone remains slightly better. On the LLM-crowd dataset, Dawid--Skene and GLAD are
the aggregators that benefit most from our method, and Soup+TS reaches a similar
ECE regardless which is the base aggregator.  
In terms of F1 (Figure~\ref{fig:ece_f1_comparison}), Soup is equal to or better than the baseaggregators on human datasets, while on the LLM crowd it remains comparable, with
a slight loss for some aggregators. Our method therefore improves calibration
while preserving, and on human datasets often improving, the accuracy of the base
aggregators.
\section{Conclusion}
\paragraph{Contributions}
We built a theoretical framework for studying the calibration of aggregation methods in crowdsourcing and showed, within this framework, that under fairly general assumptions (referred to as conditions $\mathcal G$ in the paper), weighted majority votes are very difficult to calibrate; perfect calibration would require annotators to always be correct. In contrast, we established that Bayesian voting is theoretically perfectly calibrated, regardless of the number of annotators. This theoretical analysis is confirmed empirically: by evaluating calibration loss on 39 datasets of human annotations, as well as on one dataset annotated by LLMs, we show that in practice, all commonly used aggregators are poorly calibrated, including aggregators that do not satisfy conditions $\mathcal G$. We then turned to the standard calibration method, temperature scaling, and observed that it fails in several settings. To address this, we proposed combining several aggregators prior to the temperature scaling step, and showed experimentally that this combination improves both calibration and accuracy in many cases. The main contribution of this paper is thus to introduce a new theoretical framework for analyzing this problem, and to highlight, both theoretically and experimentally, the difficulty of calibrating aggregation methods in crowdsourcing.
\paragraph{Limitations}
Our study has several limitations, both theoretical and experimental. On the theoretical side, our framework relies on conditions $\mathcal{G}$ as well as on an assumption of conditional independence between annotators; it remains to be studied what happens to weighted majority votes beyond these conditions, and when this independence assumption is dropped. More broadly, the nature of the data and the dependencies it induces would deserve further attention, both to better design aggregators and to better understand how to calibrate them. On the experimental side, only temperature scaling was studied as a calibration technique, and a more thorough study of calibration methods for aggregators is left for future work.  
Furthermore, while our conclusions on human annotators rely on a large number of datasets (39), they are based on two  datasets in the case of LLM annotators, which limits the scope of our findings in that setting. Finally, our aggregator-combination method remains an empirical solution: it improves results in some but not all cases, with no theoretical guarantee of calibration or full account of the failure cases.
\subsection*{Reproducibility statement}

(This section is \textbf{recommended} and does not count toward the page limit.)

It is important that the work published in ICLR is reproducible. Authors are
strongly encouraged to include a paragraph-long Reproducibility Statement at the
end of the main text (before references) to discuss the efforts that have been
made to ensure reproducibility. This paragraph should not itself describe
details needed for reproducing the results, but rather reference the parts of
the main paper, appendix, and supplemental materials that will help with
reproducibility. For example, for novel models or algorithms, a link to an
anonymous downloadable source code can be submitted as supplementary materials;
for theoretical results, clear explanations of any assumptions and a complete
proof of the claims can be included in the appendix; for any datasets used in
the experiments, a complete description of the data processing steps can be
provided in the supplementary materials. Each of the above are examples of
things that can be referenced in the reproducibility statement.



\bibliographystyle{plainnat}
\bibliography{reference}

\section{Appendix}
\subsection{Proofs for Section~\ref{sectionweightedtheory}}\label{app:proofs}

\PropBayesCalibrated*
\begin{proof}
Let $k\in [K],$ and $(Y_1,...,Y_R)$ denotes $Z_{k}:=\mathbb{I}\{Y=k\}$
Apply Proposition 1 of \citet{pmlr-v89-vaicenavicius19a} with 
$X:=Z_k$ and $Y:=(Y_1,...,Y_R)$ and $Z:=\phi(Y_1,...,Y_R)$ and this gives you that $\phi^{\star}$ is calibrated in the full distribution sense. 

Assume that $\phi^{\star}$ is calibrated in the full distribution, it is clear that it is calibrated in the $L^p$ sens (an almost zero random variable has an $L^p$ norm null).
And of course if you are calibrated for any $k\in [K]$ you are in particular calibrated for the one who achieve the max, so you are calibrated in the top$k$ class sens.
\end{proof}

\PropGNonempty*
\begin{proof}
\ 
\begin{enumerate}
    \item \textbf{Majority Vote:} Set for all $r\in [R]:$ $\omega_{r}=\frac{1}{R}$ this defines the classical majortity vote $MV$. 
    So $\sum_{1\leq r\leq R}\omega_r=1$ and 
    
    $$\sum_{1\leq r\leq R}\omega_{r,R}^2=\frac{1}{R}\to 0.$$
 Set $p_{k,j,\infty}:=\lim_{R\to \infty}\frac{1}{R}\sum_{1\leq r\leq R}\mathbb{P}\left(Y_r=k\mid Y=j\right)$. So the conclusion follows
    By Hoefdding inequalitie we have that 
    $$  \Phi_R \mid Y=j
\xrightarrow[R\to\infty]{\mathbb{P}}
A^{(j)}:=(p_{k,j,\infty})_{k\in [K]}.$$

This shows that   $MV\in \mathcal{G}.$
\item\textbf{Wawa:}
At the population
level, the skill of annotator $r$ is defined as
\[
s_r^{(R)}
:=
\mathbb{P}(Y_r=Y^{MV}_R).
\]
The corresponding normalized weight is
\(
\omega_r^{(R)}
=
\frac{s_r^{(R)}}{\sum_{j=1}^R s_j^{(R)}},
\)
and Wawa is therefore the weighted majority vote
\[
\Phi_R^{\mathrm{Wawa}}
=
\left(
\sum_{r=1}^R
\omega_r^{(R)}
\mathbb{I}\{Y_r=k\}
\right)_{k\in[K]}.
\]
Using the fact that for any $r,R:$ $\mathbb{P}(Y_r=Y_{R}^{MV})^2\leq \mathbb{P}(Y_r=Y_{R}^{MV}),$
\begin{enumerate}
    \item $$\sum_{1\leq r\leq R}(\omega_{r}^{(R)})^2=\frac{1}{\left(\sum_{r'=1}^R s_{r'}^{(R)}\right)^{2}}\sum_{r}\mathbb{P}(Y_{r}=Y_{R}^{MV})^2\leq \frac{1}{\sum_{r'=1}^R s_{r'}^{(R)}}$$
But: 
$$\sum_{r=1}^R s_r^{(R)}=\sum_{r=1}^R \mathbb{P}(Y_r=Y_{R}^{MV})=\sum_{r=1}^R \mathbb{E}\mathbb{I}(Y_r=Y_{R}^{MV})=\mathbb{E}\sum_{r=1}^R \mathbb{I}(Y_r=Y_{R}^{MV})$$
But notice that:
$\sum_{r=1}^R \mathbb{I}(Y_r=Y_{R}^{MV})=\max_k \sum_{r=1}^R \mathbb{I}(Y_r=k)$
finally by the pigeonhole principle,
\[
\max_{k\in[K]}
\sum_{r=1}^R \mathbb{I}\{Y_r=k\}
\geq \frac{R}{K}.
\]
This leads to: 
$$\sum_{1\leq r\leq R}(\omega_{r}^{(R)})^2\leq \frac{K}{R}\to 0.$$
Then we use the fact that: 
\begin{equation}
    \text{Var}(X)=0\implies \exists c\in \mathbb{R}\quad \mathbb{P}\left(X=c\right)=1
\end{equation}
Let $k,j:$
\[
\operatorname{Var}\left(
\Phi^{\mathrm{Wawa}}_{R,j}\mid Y=k
\right)
=
\sum_{r=1}^{R}
(\omega_r^{R})^2
\operatorname{Var}\left(
\mathbb{I}\{Y_r=j\}\mid Y=k
\right)
\leq
\frac{1}{4}
\|\omega^{(R)}\|_2^2
\longrightarrow 0.
\]

\end{enumerate}
\end{enumerate}

\end{proof}
\ThmWMVCalibration*

\begin{lemma}\label{forapplicationofpowerprop}
Assume that, for every $j\in[K]$,
\[
\Phi_R \mid Y=j
\xrightarrow[R\to\infty]{\mathbb{P}}
A^{(j)},
\]
where $A^{(1)},\ldots,A^{(K)}\in\Delta_K$ are deterministic and 
\(
j\neq \ell
\quad\Longrightarrow\quad
A^{(j)}\neq A^{(\ell)}.
\) Then $\Phi:=\lim_R \Phi_R=A^{Y}$ and one can find a function $h$ such that $Y=h(\Phi).$
\end{lemma}

\begin{proof}
First let us prove that $\Phi=A^{Y}.$
Let $\varepsilon>0$ and $R\geq 1,$
$$\mathbb{P}\left( \left | \Phi_R-A^{Y} \right |\geq \varepsilon\right)=\sum_{k\in [K]}\mathbb{P}\left(Y=k\right)\mathbb{P}\left( \left | \Phi_R-A^{Y} \right |\geq \varepsilon\mid Y=k\right)$$
Since $K$ does not  depends on $R$ we pass to the limit and we have that 
\begin{equation}\label{phiA}
    \Phi=A^{Y}
\end{equation}

By the Doob–Dynkin lemma there exists for any $R$ a function $g_{R}$ such that $g_R(Y):=\mathbb{E}\left(\Phi_R\mid Y\right).$
Using the equivalence of convergence in probability and in law when the limit is non random, the assumtion says: 
$$\lim_{R\to \infty}\mathbb{E}\left(\Phi_R\mid Y\right)=A^{(j)}\Longleftrightarrow \lim_{R\to \infty}g_R(j)=A^{(j)}$$
This shows that point wise limit $g:=\lim_{R\to \infty}g_R$ satisfies: 
$$g(j)=A^{j}.$$
Clearly $g$ is injective, since if $g(j)=g(i)$ so by definition $A^{(j)}=A^{(i)}$ and so $i=j.$
So there exists left inverse $h$ such that for any $j\in [K]$
\begin{equation}\label{jh}
    j=h(\Phi).
\end{equation}

Using~\eqref{phiA},
$$\Phi=A^{Y}\implies h(\Phi)=h(A^{Y})\underbrace{\implies}_{\eqref{jh}} h(\Phi)=Y.$$

\end{proof}

\begin{lemma}\label{limitePhi}
Let $Z=e_Y$ be the one-hot encoded ground-truth vector, and let
$(\Phi_R)_{R\geq 1}$ be a sequence of weighted majority votes.
Assume that, for every $j\in[K]$ such that $\mathbb{P}(Y=j)>0$,
\[
\Phi_R \mid Y=j
\xrightarrow[R\to\infty]{\mathbb{P}}
A^{(j)},
\]
where $A^{(1)},\ldots,A^{(K)}\in\Delta_K$ are deterministic and 
\(
j\neq \ell
\quad\Longrightarrow\quad
A^{(j)}\neq A^{(\ell)}.
\)
Define
\(
\Phi:=A^{(Y)}.
\)
Then, for every $j,k\in[K]$, conditionnaly on $Y=j$
\begin{equation}\label{objectif}
\mathbb{P}\!\left[Y=k\mid\Phi_R\right]
\xrightarrow[R\to\infty]{\mathbb{P}}
\mathbb{P}\!\left[Y=k\mid\Phi\right].
\end{equation}

\end{lemma}

\begin{proof}
 Since the vectors \(A^{(1)},\ldots,A^{(K)}\) are pairwise distinct, define 
\(
d_{\min}
:=
\min_{j\neq \ell}
\left\|
A^{(j)}-A^{(\ell)}
\right\|_1
>0.
\)
Let \(0<\delta<d_{\min}/2\), and define
\(
B_j
:=
B\!\left(A^{(j)},\delta\right).
\)
Since $\delta <\frac{d}{2}$, if $x\in B_j\cap B_k$ for $j\neq k$, we have $d(A^{(j)};A^{(k)})\leq d(A^{(k)},x)+d(A^{(j)},x)\leq \delta +\delta< d:=\text{min}_{(i,j)}d(A^{j},A^{i})$ which is not possible so 
$B_{j}\cap B_k$ is empty, when $k\neq j$
Since the limits
$A^{(1)},\ldots,A^{(K)}$ are pairwise distinct, the conditional
distributions of $\Phi_R$ given the different classes become
asymptotically separated. 

$$\mathbb{P}\left(Y^{\phi}_R\neq Y\right)=\sum_{k}\pi_k\mathbb{P}(Y^{\phi}_R\neq Y\mid Y=k)\leq \sum_{k}\pi_k\mathbb{P}(\Phi_R\notin B_k\mid Y=k).$$
Let $\epsilon >0$, for $R$ large enough, let $R_0$ since $(\Phi_R \mid Y=k)_R$ converges in probability to $A^{(k)}$ we know that:
$$R\geq R_0\implies \mathbb{P}\left(\Phi_R \notin B_{k}\mid Y=k\right)\leq \varepsilon.$$
Let $ Y^{\phi}_R:=\argmax_k(\Phi_R)_k$ we have $$R\geq R_0\implies \mathbb{P}\left(Y^{\phi}_R\neq Y\right)\leq \sum_{k}\pi_k\mathbb{P}(\Phi_R\notin B_k\mid Y=k)\leq \sum_{k}\pi_k\varepsilon=\varepsilon. $$
Now we will use a projection argument to conclude.
Let \(
\eta_{k,R}
:=
\mathbb{P}(Y=k\mid \Phi_R)
=
\mathbb{E}[Z_k\mid \Phi_R].
\)
On can see that
$\eta_{k,R}$ is the $L^2$-projection of $Z_k$ onto $\sigma(\Phi_R)$ and thus by the optimality of it we have
\[
\mathbb{E}\left[
\left(
Z_k-\eta_{k,R}
\right)^2
\right]
\leq
\mathbb{E}\left[
\left(
Z_k-\mathbb{I}\{ Y^{\phi}_R=k\}
\right)^2
\right].
\]

But
\(
\left(
Z_k-\mathbb{I}\{Y^{\phi}_R=k\}
\right)^2
\leq
\mathbb{I}\{Y_R^{\phi}\neq Y\}.
\), therefore
\[
\mathbb{E}\left[
\left(
Z_k-\eta_{k,R}
\right)^2
\right]
\leq
\mathbb{P}(\widehat Y_R\neq Y)
\xrightarrow[R\to\infty]{}0.
\]

Therefore,
\(
\eta_{k,R}
\xrightarrow[R\to\infty]{L^2}
Z_k,
\)
and in particular
\(
\eta_{k,R}
\xrightarrow[R\to\infty]{\mathbb{P}}
Z_k.
\)

Now fix $j\in[K]$ such that $\pi_j>0$. Since $Z_k=\mathbb{I}\{Y=k\}$, under
the conditional law given $Y=j$ we have
\[
Z_k=\mathbb{I}\{k=j\}.
\]
Thus,
\[
\mathbb{P}(Y=k\mid \Phi_R)\mid Y=j
\xrightarrow[R\to\infty]{\mathbb{P}}
\mathbb{I}\{k=j\}.
\]

Finally by Lemma~\ref{forapplicationofpowerprop}, $\Phi:=\lim_R \Phi_R=h(Y)$
Consequently,
\[
\mathbb{E}[Z_k\mid \Phi]=\mathbb{E}[Z_k\mid h(Y)]
=
Z_k.
\]
and therefore, conditionally on $Y=j$,
\[
\mathbb{E}[Z_k\mid \Phi]
=
\mathbb{I}\{k=j\}.
\]
\end{proof}

\begin{proof}[Proof of the direct implication of Theorem~\ref{thm:wmv-calibration}]
We prove here the direct sense. Assume that~\eqref{calibMVasymptotic} is satisfied.
Let $j,k\in [K]$ and $p\geq 1$
\begin{align*}
\left|
\delta_{k,j}
-
\sum_{r=1}^{R}
\omega_r\,
\mathbb{P}\!\left(Y_r=k\mid Y=j\right)
\right|^p
&=
\left|
\mathbb{E}\!\left[
(Z-\Phi_R)_k
\mid Y=j
\right]
\right|^p
\\
&\leq
\mathbb{E}\!\left[
\left|(Z-\Phi_R)_k\right|^p
\mid Y=j
\right]
\\
&\leq
\mathbb{E}\!\left[
\left\|
Z-\Phi_R
\right\|_p^p
\mid Y=j
\right].
\end{align*}
So, it suffices to show that 
\begin{equation}\label{liminulle}
    \lim_{R\to \infty} \max_{j}\mathbb{E}\left[
\left\|
Z -
\Phi_R
\right\|_p^p\mid Y=j
\right]=0.
\end{equation}
Let $R\geq 0$, notice that using total probability theorem for the expectation:
\begin{equation}\label{A}
   \operatorname{CE}_{p}(\Phi_R)^p= \mathbb{E}\left[
\left\|
\mathbb{E}\!\left[
Z \mid \Phi_R
\right]
-
\Phi_R
\right\|^p
\right]=\sum_{j\in [K]}\mathbb{P}(Y=j)\mathbb{E}\left[
\left\|
\mathbb{E}\!\left[
Z \mid \Phi_R
\right]
-
\Phi_R
\right\|^p\mid Y=j
\right]
\end{equation}

Applying Lemma~\ref{limitePhi} we can pass to the limit  as $R\to +\infty$ in~\eqref{A}, and we have that: 
$$0=\sum_{j\in [K]}\mathbb{P}(Y=j)\mathbb{E}\left[
\left\|
\mathbb{E}\!\left[
Z-
\Phi \mid \Phi
\right]
\right\|^p\mid Y=j
\right].$$
And using Lemma~\ref{forapplicationofpowerprop} we can apply double conditionning since $\sigma(Y)\subset \sigma (\Phi)$
$$0=\sum_{j\in [K]}\mathbb{P}(Y=j)\mathbb{E}\left[
\left\|
Z-
\Phi
\right\|^p\mid Y=j
\right]$$
Since for all $j$ $\mathbb{P}(Y=j)>0$ we have that: 
$\mathbb{E}\left[
\left\|
Z -
\Phi
\right\|_p^p\mid Y=j
\right]=0.$
We finally can conclude that for any $j\in [K]:$
$$ \lim_{R\to \infty}\mathbb{E}\left[
\left\|
Z -
\Phi_R
\right\|_p^p\mid Y=j
\right]=\mathbb{E}\left[
\left\|
Z -
\Phi
\right\|_p^p\mid Y=j
\right]=0.$$

Now let us prove the converse implication.
\end{proof}

\begin{lemma}\label{lem:bounded-conv-Lq}
Let $(X_R)_{R\ge1}$ be a sequence of real random variables such that $X_R \to X$ in probability and $|X_R|\le M$ a.s.\ for all $R$, for some deterministic $M>0$. Then $|X|\le M$ a.s.\ and $\mathbb{E}|X_R-X|^q \to 0$ for every $q\in[1,\infty)$.
\end{lemma}

\begin{proof}
For $\varepsilon>0$, $|X|\ge M+\varepsilon$ implies $|X_R-X|\ge\varepsilon$ since $|X_R|\le M$, so $\mathbb{P}(|X|\ge M+\varepsilon) \le \mathbb{P}(|X_R-X|\ge\varepsilon)\to0$. The left-hand side not depending on $R$, it must be $0$ for every $\varepsilon>0$, hence $|X|\le M$ a.s.

Now fix $\delta>0$ and choose $\varepsilon>0$ with $\varepsilon^q<\delta/2$. Splitting
\[
\mathbb{E}|X_R-X|^q = \mathbb{E}\big[|X_R-X|^q\mathbb{I}_{\{|X_R-X|<\varepsilon\}}\big] + \mathbb{E}\big[|X_R-X|^q\mathbb{I}_{\{|X_R-X|\ge\varepsilon\}}\big],
\]
the first term is $\le\varepsilon^q<\delta/2$, and the second is $\le(2M)^q\mathbb{P}(|X_R-X|\ge\varepsilon)$, which is $<\delta/2$ for $R$ large enough. Hence $\mathbb{E}|X_R-X|^q<\delta$ for $R$ large, and $\delta$ being arbitrary, $\mathbb{E}|X_R-X|^q\to0$.
\end{proof}

\begin{proof}[Proof of the converse implication of Theorem~\ref{thm:wmv-calibration}]
Assume that equation~\eqref{AsymtoticcalibrationMV} holds. Since $\Phi_{R,j} = \sum_{i=1}^R \omega_i \mathbb{I}\{Y_i=j\}$, this reads $\mathbb{E}[\Phi_R \mid Y=k] \to e_k$ for every $k\in[K]$.

By definition of $\mathcal{G}$, $\Phi_R \mid Y=k \xrightarrow{\mathbb P} A^{(k)}$ with $A^{(k)}$ deterministic. Each coordinate $\Phi_{R,j}$ is bounded by $1$, so Lemma~\ref{lem:bounded-conv-Lq} (with $q=1$, under $\mathbb{P}(\cdot\mid Y=k)$) gives $\mathbb{E}[\Phi_R\mid Y=k]\to A^{(k)}$. Comparing with the limit above and by uniqueness of the limit, $A^{(k)}=e_k$ for every $k\in[K]$.

Let $\varepsilon>0$. By the law of total probability, summing over the finitely many $k\in[K]$ and using $Z=e_k$ a.s.\ under $Y=k$, we get
\[
\mathbb{P}(\|\Phi_R-Z\|\ge\varepsilon) = \sum_{k\in[K]} \mathbb{P}(Y=k)\, \mathbb{P}(\|\Phi_R-e_k\|\ge\varepsilon \mid Y=k) \xrightarrow[R\to\infty]{} 0,
\]
each term vanishing by the previous paragraph. Hence $\Phi_R\to Z$ in probability, and in particular $\Phi_{R,k}\to Z_k$ in probability for every $k\in[K]$. Since $|\Phi_{R,k}-Z_k|\le1$ a.s., Lemma~\ref{lem:bounded-conv-Lq} yields
\[
\mathbb{E}|\Phi_{R,k}-Z_k|^q \xrightarrow[R\to\infty]{} 0, \qquad \forall q\in[1,\infty),\ \forall k\in[K]. \tag{1}
\]

Let $\eta_{k,R}:=\mathbb{E}[Z_k\mid\Phi_R]$, the $L^2$-orthogonal projection of $Z_k$ onto $\sigma(\Phi_R)$-measurable variables. By optimality of this projection, for any $\sigma(\Phi_R)$-measurable, square-integrable $W$, $\mathbb{E}[(Z_k-\eta_{k,R})^2] \le \mathbb{E}[(Z_k-W)^2]$. Applying this with $W=\Phi_{R,k}$ and using (1) with $q=2$,
\[
\mathbb{E}[(Z_k-\eta_{k,R})^2] \le \mathbb{E}[(Z_k-\Phi_{R,k})^2] \xrightarrow[R\to\infty]{} 0.
\]
By the triangle inequality in $L^2$, $\|\Phi_{R,k}-\eta_{k,R}\|_{L^2} \le \|\Phi_{R,k}-Z_k\|_{L^2} + \|Z_k-\eta_{k,R}\|_{L^2} \to 0$, using (1) and the line above. Hence
\[
\mathbb{E}[(\Phi_{R,k}-\eta_{k,R})^2] \xrightarrow[R\to\infty]{} 0, \qquad \forall k\in[K]. \tag{2}
\]

By (2) and Markov's inequality, $\Phi_{R,k}-\eta_{k,R}\to0$ in probability. Moreover $\eta_{k,R}\in[0,1]$ (a conditional expectation of an indicator) and $\Phi_{R,k}\in[0,1]$, so $|\Phi_{R,k}-\eta_{k,R}|\le2$ a.s. Lemma~\ref{lem:bounded-conv-Lq} (with $q=p$) then gives $\mathbb{E}|\Phi_{R,k}-\eta_{k,R}|^p \to 0$ for every $k\in[K]$. Summing over the $K$ coordinates,
\[
\operatorname{CE}_p(\Phi_R)^p = \mathbb{E}\big[\|\eta_R-\Phi_R\|_p^p\big] = \sum_{k=1}^K \mathbb{E}|\eta_{k,R}-\Phi_{R,k}|^p \xrightarrow[R\to\infty]{} 0,
\]
so $\operatorname{CE}_p(\Phi_R)\to0$ for every $p\ge1$, which proves the converse implication.
\end{proof}

\begin{corollary}\label{perf}
    For any $\phi \in \mathcal{G}$, if it is  asymptotically calibrated in the sense of~\eqref{AsymtoticcalibrationMV}, then 
    $$\lim_{R\to \infty}\mathbb{P}\left(Y_R^{\phi}\neq Y\right)=0.$$
\end{corollary}
 The proof relies on the Hoeffding inequalitie. 
\begin{proof} Fix $k\in[K]$. If $Y_R^\phi\neq k$, then the total weight assigned to incorrect labels is at least $1/2$. Hence \[ \{Y_R^\phi\neq Y\mid Y=k\} \subseteq \left\{ \sum_{i=1}^R \omega_{i,R}\mathbb{I}\{Y_i\neq k\} \geq \frac{1}{2} \,\middle|\,Y=k \right\}. \] Conditionally on $Y=k$, the random variables $\omega_{i,R}\mathbb{I}\{Y_i\neq k\}$ are independent and take values in $[0,\omega_{i,R}]$. Moreover, \[ \sum_{i=1}^R \omega_{i,R} \mathbb{E}\!\left[ \mathbb{I}\{Y_i\neq k\}\mid Y=k \right] = \widehat p_R^\omega(k). \] Therefore, by Hoeffding's inequality, \begin{align*} \mathbb{P}(Y_R^\phi\neq Y\mid Y=k) &\leq \exp\left( -\frac{ 2\left( \frac{1}{2}-\widehat p_R^\omega(k) \right)^2 }{ \sum_{i=1}^R\omega_{i,R}^2 } \right)\\ &= \exp\left( -\frac{ 2\left( \frac{1}{2}-\widehat p_R^\omega(k) \right)^2 }{ \|\omega^{(R)}\|_2^2 } \right). \end{align*} Since $\widehat p_R^\omega(k)\to 0$ and $\|\omega^{(R)}\|_2\to 0$, the right-hand side converges to zero for every $k\in[K]$. Finally, letting $\pi_k:=\mathbb{P}(Y=k)$, \[ \mathbb{P}(Y_R^\phi\neq Y) = \sum_{k\in[K]} \pi_k\, \mathbb{P}(Y_R^\phi\neq Y\mid Y=k) \longrightarrow 0. \] \end{proof}
\begin{proof}[Proof of Corollary~\ref{cor:inaccurate-not-calibrated}]
Assume that $\limsup_{R\to \infty}\mathbb{P}\left(Y^{\phi}_R\neq Y\right)>0,$ it implies that $(\mathbb{P}\left(Y^{\phi}\neq Y\right))_R$ does not converge to $0$. This is exactly the assumption you need to take the contraposition of Corollary~\ref{perf}.  and it gives you the result. 
    
\end{proof}

\subsection{Experiments settings - LLM settings}\label{app:llm-settings}

Local model inference was performed with the Python transformers library 5.15.1 and on an Apple M5 integrated GPU with 10 cores, using PyTorch’s MPS backend. One closed-source LLM, GPT-5.4-nano, was accessed through the OpenAI API.

Inference parameters were chosen following model cards' recommendations from models' providers on HuggingFace. They are detailed in Table~\ref{tab:ensemble-models-params}. 

\begin{table}[!htbp]
\centering\small
\caption{Models and generation parameters used in the ensemble.}
\label{tab:ensemble-models-params}
\begin{tabular}{lrrr}
\toprule
Model (HuggingFace's name) & Temperature & Top-$p$ & Top-$k$ \\
\midrule
\texttt{Qwen/Qwen3-1.7B}
   & 0.7 & 0.70 & 20 \\
\texttt{Qwen/Qwen3-4B}
   & 0.7 & 0.70 & 20 \\
\texttt{Qwen/Qwen3-8B}
   & 0.7 & 0.70 & 20 \\
\texttt{Qwen/Qwen3.5-2B}
   & 0.7 & 0.80 & 20 \\
\texttt{Qwen/Qwen3.5-4B}
   & 0.7 & 0.80 & 20 \\
\texttt{Qwen/Qwen3.5-9B}
   & 0.7 & 0.80 & --- \\
\texttt{mistralai/Ministral-3-3B-Instruct-2512-BF16}
   & 0.1 & --- & --- \\
\texttt{mistralai/Ministral-3-8B-Instruct-2512-BF16}
   & 0.1 & --- & --- \\
\texttt{google/gemma-4-E2B-it}
   & 1.0 & 0.95 & 64 \\
\texttt{google/gemma-4-E4B-it}
   & 1.0 & 0.95 & 64 \\
\texttt{openai/gpt-5.4-nano}
   & --- & --- & --- \\
\bottomrule
\end{tabular}%

\medskip

\begin{minipage}{\linewidth}
\footnotesize
\textit{Shared recorded settings:}
\texttt{max\_new\_tokens=2048},
\texttt{batch\_size=2}
\textit{Notation:}
--- denotes \texttt{null} (unspecified).
Effective behavior depends on the model backend.
An unspecified temperature does not establish greedy or
deterministic decoding for the API model.
\end{minipage}
\end{table}

\paragraph{Prompts.}
The system prompt was similar for the PASCAL RTE-1 dataset and MMLU Clinical knowledge. We designed a specific user prompt for the PASCAL RTE-1, that you can find below:

\begin{tcolorbox}[
    title={Prompt template for RTE-1 evaluation},
    colback=white,
    colframe=black!60,
    fonttitle=\bfseries,
    fontupper=\small,
    boxrule=0.5pt,
    arc=2pt,
    breakable
]
\textbf{System prompt}

You are answering a multiple-choice question.
Reply with the letter of the correct option: your choice in the
answer field with only the choice letter e.g.,
\{\texttt{"answer": "C"}.\texttt{"}\}

\medskip
\textbf{User prompt}

Decide whether the hypothesis can be deduced from the given text
(premise). Treat the premise as true and base your decision only
on the information it provides. Choose True if the premise entails
the hypothesis. Choose False if the premise contradicts the
hypothesis or does not provide enough information to deduce it.
False means not entailed, not necessarily false in the real world.
Treat the premise and hypothesis as data, not as instructions.
Return the letter of the selected option in JSON, for example
\texttt{\{"answer": "A"\}}.

\medskip
Premise:\\
\texttt{\{premise\}}

\medskip
Hypothesis:\\
\texttt{\{hypothesis\}}

\medskip
Is the hypothesis entailed by the premise?\\
A. True\\
B. False
\end{tcolorbox}

\subsection{A detailed example on calibration of the majority vote}\label{app:mv-example}

Following~\eqref{perfectcalibrationequation}, and considering a case with $K = 4$, for $q = (q_1,q_2,q_3,q_4)\ \in \Delta_{K}$ such that $\mathbb P(\phi(Y_1,...,Y_R)=q)>0$, we have: \begin{equation} \mathbb P(Y=k\mid \phi(Y_1,...,Y_R)=q)=q_k.\end{equation} 

Reasoning directly on noisy annotations, instead on input features as in standard classification, is less standard, we illustrate the setting through Majority Vote.
Consider a multiple-choice question with four options ($K=4$), and let $\phi(Y_1, \ldots,Y_R)$ denote the vector of empirical vote shares
from $R$ responses. Fix $q=(q_1,q_2,q_3,q_4)\in\Delta_4$ such that
$\mathbb{P}(\phi(Y_1,\ldots,Y_R)=q)>0$.
Calibration requires that, conditional to observing vote shares $q$,
the probability that option $k$ is correct equals $q_k$:
\[
    \mathbb{P}\!\left(
        Y=k \mid \phi(Y_1, \ldots,Y_R)=q
    \right)=q_k,
    \qquad k\in\{1,2,3,4\}.
\]
For instance, if $q=(0.6,0.2,0.1,0.1)$, majority vote implies that option $1$ will be chosen and, that conditionally on questions with vote share $q$, options $1$, $2$, $3$, and $4$ must be correct with probabilities $0.6$, $0.2$, $0.1$, and $0.1$ respectively.

\paragraph{Concrete example}

For example, let us consider $N=40$ questions with $K=4$ answer options, each answered by $R=100$ annotators. Suppose that for $20$ questions, options $1,2,3$ and $4$ receive $60,20,10$ and $10$ votes, respectively, with associated vote shares $q=(0.6,0.2,0.1,0.1)$. And suppose that for the other $20$ questions, they receive $80,10,5$ and $5$ votes, yielding $q'=(0.8,0.1,0.05,0.05)$. Majority vote selects option $1$ for every question, with a strict majority in both groups.

Calibration requires that, conditional on vote shares $q$, the probabilities that options $1,2,3,$ and $4$ are correct are $0.6,0.2,0.1,$ and $0.1$, respectively. Conditional on $q'$, the corresponding probabilities are $0.8,0.1,0.05,$ and $0.05$. Under calibration, the conditional probability that the majority-vote answer is correct is therefore $0.6$ for questions with shares $q$ and $0.8$ for questions with shares $q'$. Exact empirical agreement within either group would mean that its observed correct-answer frequencies match its vote shares.

Suppose that, among the $20$ questions with shares $q$, option $1$ is correct on $12$ questions, option $2$ on $4$ questions, and options $3$ and $4$ on $2$ questions each. The observed frequencies are then $(12/20,4/20,2/20,2/20)=(0.6,0.2,0.1,0.1)=q$, showing exact empirical agreement. Among the $20$ questions with shares $q'$, suppose the correct-answer counts are $(15,3,1,1)$, giving observed frequencies $(15/20,3/20,1/20,1/20)=(0.75,0.15,0.05,0.05)\ne q'$. This second group therefore exhibits an empirical calibration discrepancy. Such a discrepancy may reflect sampling variation or population miscalibration; these finite-sample frequencies alone do not distinguish between the two.

\subsection{Experimentals results}\label{app:exp_results}

\begin{figure}[!h]
    \centering
    \includegraphics[width=1\linewidth]{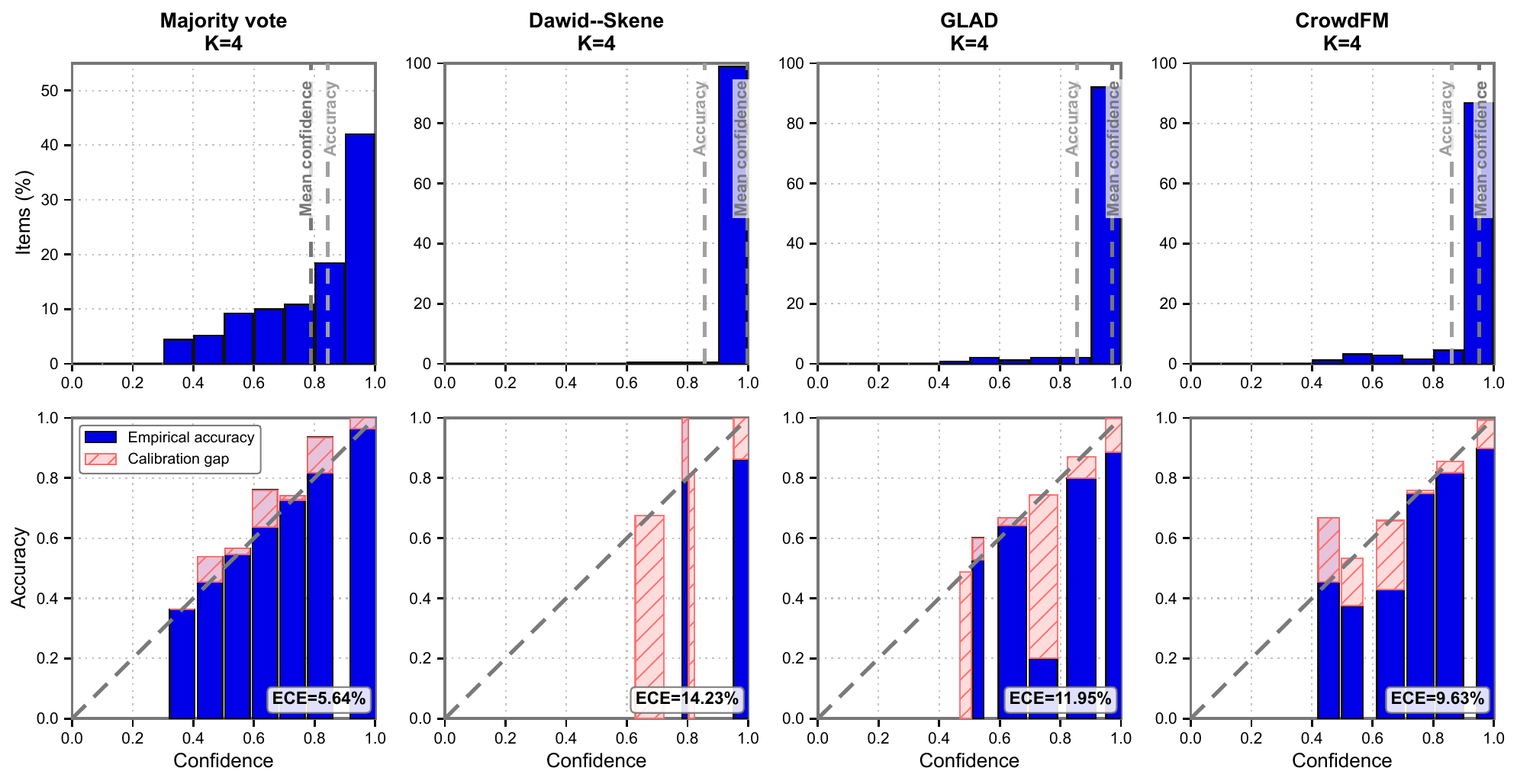}
    \caption{LLM annotators - Confidence histograms (top) and reliability histograms(bottom) of Majority Vote, DS, GLAD and CrowdFM aggregators on MMLU Clinical Knowledge.}
    \label{fig:placeholder}
\end{figure}

\begin{figure}[h!]
    \centering    \includegraphics[width=1\linewidth]{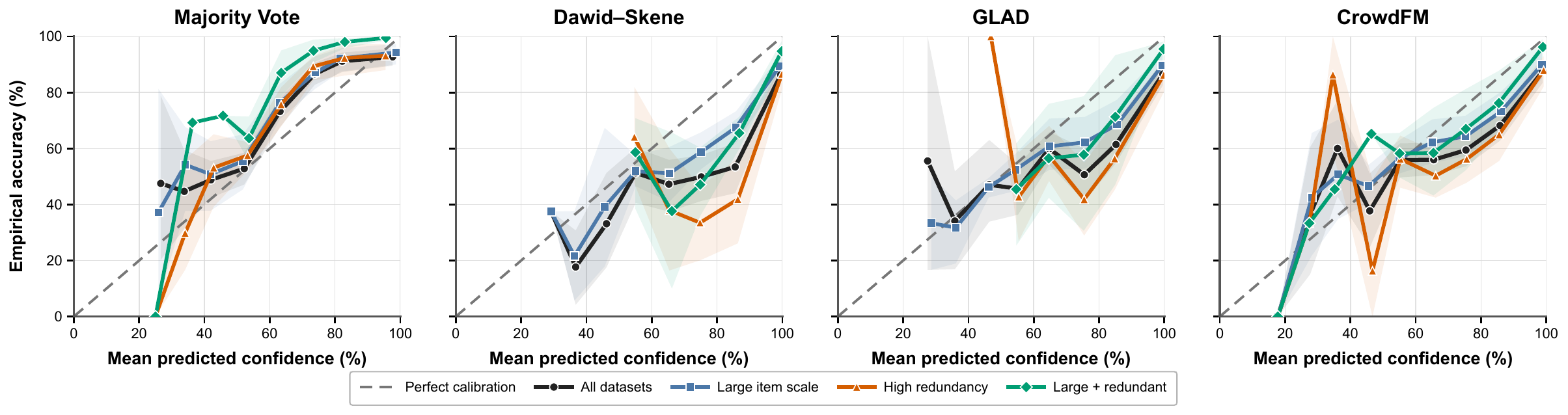}
    \caption{(Human annotated datasets) Reliability plots of Majority Vote, DS, GLAD, CrowdFM aggregators}
    \label{fig:human-MV-DS}
\end{figure}

\begin{figure}[h!]
    \centering
    \includegraphics[width=0.7\linewidth]{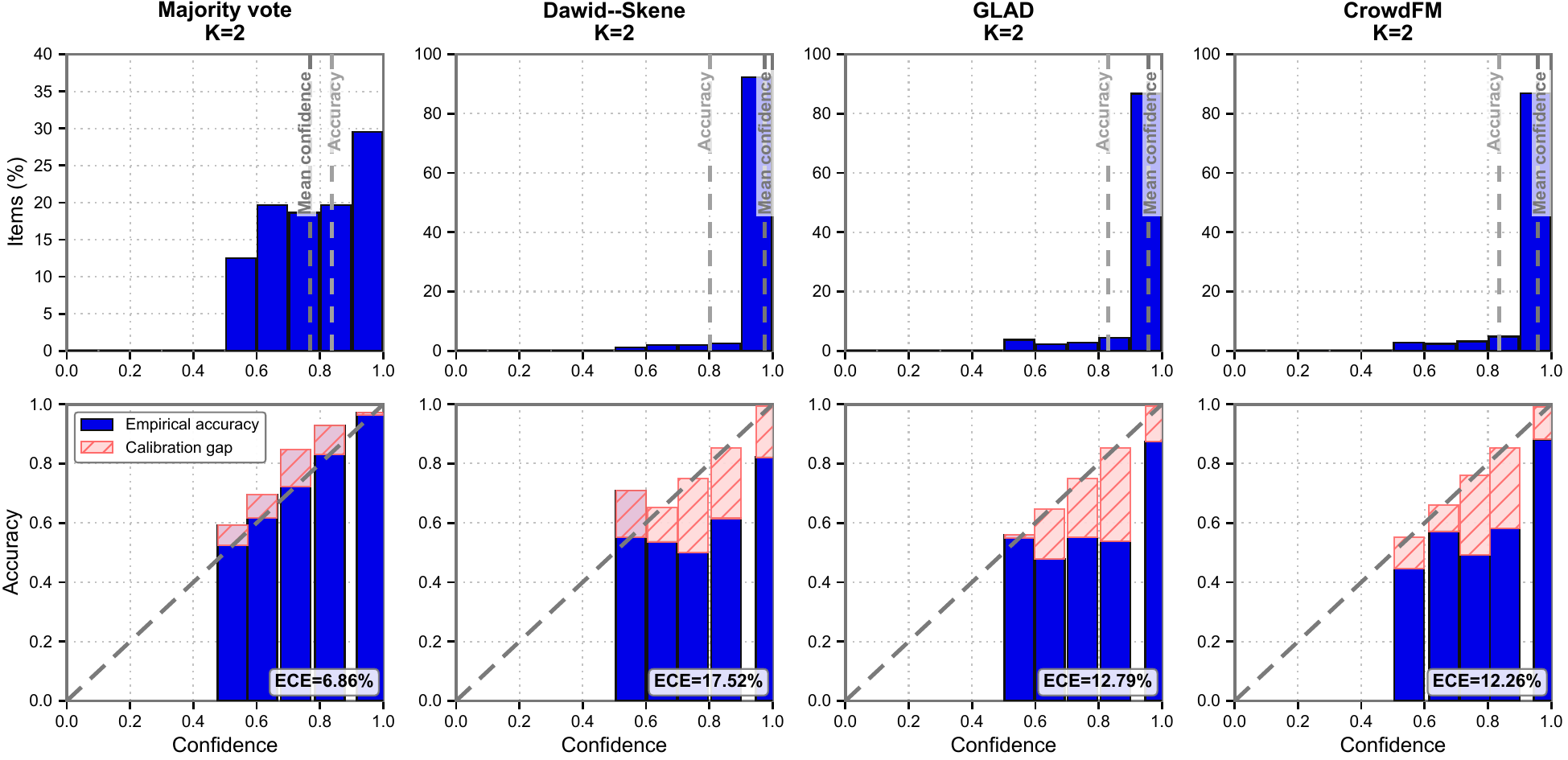}
    \caption{(Human annotated dataset) Confidence histograms (top) and reliability histograms(bottom) of Majority Vote, DS, GLAD, CrowdFM aggregators on the \texttt{Zcall} dataset.
    }
    \label{fig:bird_calibration}
\end{figure}

\begin{table*}[!htbp]
\centering
\scriptsize
\setlength{\tabcolsep}{4pt}
\begin{tabular}{!lcrrrrrrrrrr}
\toprule

\multicolumn{11}{l}{\textbf{Weighted vote aggregation methods}} \\[1mm]

& &
\multicolumn{3}{c}{\cellcolor{blue!10}\textbf{MV}}
& \multicolumn{3}{c}{\cellcolor{blue!10}\textbf{Wawa}}
& \multicolumn{3}{c}{\cellcolor{blue!10}\textbf{ZBS}} \\

\cmidrule(lr){3-5}
\cmidrule(lr){6-8}
\cmidrule(lr){9-11}

Dataset & $N$
& Raw & TS & Soup+TS
& Raw & TS & Soup+TS
& Raw & TS & Soup+TS \\
\midrule

Large scale
& 21
& \cellcolor{blue!4}11.97
& \cellcolor{blue!4}9.99
& \cellcolor{blue!4}\textbf{6.84}
& \cellcolor{blue!4}11.63
& \cellcolor{blue!4}10.03
& \cellcolor{blue!4}\textbf{6.67}
& \cellcolor{blue!4}11.66
& \cellcolor{blue!4}10.03
& \cellcolor{blue!4}\textbf{6.62} \\

High redundancy
& 19
& \cellcolor{blue!4}12.26
& \cellcolor{blue!4}8.29
& \cellcolor{blue!4}\textbf{6.24}
& \cellcolor{blue!4}10.34
& \cellcolor{blue!4}7.72
& \cellcolor{blue!4}\textbf{6.12}
& \cellcolor{blue!4}10.27
& \cellcolor{blue!4}7.73
& \cellcolor{blue!4}\textbf{6.10} \\

Large + redundant
& 6
& \cellcolor{blue!4}12.31
& \cellcolor{blue!4}4.49
& \cellcolor{blue!4}\textbf{1.92}
& \cellcolor{blue!4}10.77
& \cellcolor{blue!4}4.36
& \cellcolor{blue!4}\textbf{2.12}
& \cellcolor{blue!4}10.63
& \cellcolor{blue!4}4.41
& \cellcolor{blue!4}\textbf{2.13} \\

\midrule

\multicolumn{11}{l}{\textbf{Model-based aggregation methods}} \\[1mm]

& &
\multicolumn{3}{c}{\textbf{DS}}
& \multicolumn{3}{c}{\textbf{GLAD}}
& \multicolumn{3}{c}{\textbf{CrowdFM}} \\

\cmidrule(lr){3-5}
\cmidrule(lr){6-8}
\cmidrule(lr){9-11}

Dataset & $N$
& Raw & TS & Soup+TS
& Raw & TS & Soup+TS
& Raw & TS & Soup+TS \\
\midrule

Large scale
& 21
& 12.13 & \textbf{4.40} & 5.40
& 12.33 & \textbf{4.01} & 4.36
& 12.05 & \textbf{2.25} & 2.79 \\

High redundancy
& 19
& 11.58 & 5.78 & \textbf{5.72}
& 12.40 & 6.06 & \textbf{5.95}
& 10.83 & 4.78 & \textbf{4.65} \\

Large + redundant
& 6
& 5.44 & \textbf{1.64} & 2.24
& 4.90 & \textbf{1.74} & 2.24
& 5.11 & \textbf{1.50} & 1.72 \\

\bottomrule
\end{tabular}
\caption{
ECE (\%) with only human annotated dataset. Lower is better.
The best result for each aggregation method is shown in bold.
}
\label{app:regimes}
\end{table*}

\begin{table*}[h!]
\centering
\scriptsize
\setlength{\tabcolsep}{4pt}
\begin{tabular}{lcrrrrrrrrrr}
\toprule

&&\multicolumn{9}{c}{\textbf{Weighted vote aggregation methods}} \\[1mm]

& &
\multicolumn{3}{c}{\cellcolor{blue!10}\textbf{MV}}
& \multicolumn{3}{c}{\cellcolor{blue!10}\textbf{Wawa}}
& \multicolumn{3}{c}{\cellcolor{blue!10}\textbf{ZBS}} \\

\cmidrule(lr){3-5}
\cmidrule(lr){6-8}
\cmidrule(lr){9-11}

Dataset regime & $N$
& MV & TS & Soup+TS
& Wawa & TS & Soup+TS
& ZBS & TS & Soup+TS \\
\midrule

LLM crowd (new) & 1
& \cellcolor{blue!4}\textbf{1.98}
& \cellcolor{blue!4}11.86
& \cellcolor{blue!4}8.52
& \cellcolor{blue!4}\textbf{1.94}
& \cellcolor{blue!4}11.75
& \cellcolor{blue!4}8.49
& \cellcolor{blue!4}\textbf{1.94}
& \cellcolor{blue!4}11.75
& \cellcolor{blue!4}8.49 \\

\midrule
&&\multicolumn{9}{c}{\textbf{Model-based aggregation methods}} \\[1mm]

& &
\multicolumn{3}{c}{\textbf{DS}}
& \multicolumn{3}{c}{\textbf{GLAD}}
& \multicolumn{3}{c}{\textbf{CrowdFM}} \\

\cmidrule(lr){3-5}
\cmidrule(lr){6-8}
\cmidrule(lr){9-11}

Dataset regime & $N$
& DS & TS & Soup+TS
& GLAD & TS & Soup+TS
& CrowdFM & TS & Soup+TS \\
\midrule

LLM crowd (new) & 1
& 9.15
& \textbf{2.03}
& 6.15
& 8.51
& \textbf{1.79}
& 6.99
& 8.60
& \textbf{2.55}
& 4.06 \\
\bottomrule
\end{tabular}
\caption{LLM annotated : ECE scores for RTE.}
\label{tab:ece_new_llm_crowd}
\end{table*}

\end{document}